\documentclass{article}

\usepackage{hyperref}
\usepackage[dvipsnames]{xcolor}
\hypersetup{
   colorlinks=true,
   linkcolor=MidnightBlue,
   citecolor=Green
}

\usepackage{amsthm,amsmath,amssymb,amscd,mathrsfs}

\usepackage{enumitem}
\usepackage{nicefrac}

\usepackage[capitalize,noabbrev]{cleveref}

\usepackage{crossreftools}
\usepackage{xspace}
\newcommand{\RAPP}{{\hyperref[alg:inexactResolvent]{RAPP}}\xspace}

\usepackage{xargs} 
\usepackage{xstring}
\usepackage{etoolbox}

\usepackage{etoc}
\usepackage[bibliography=common]{apxproof}

\renewcommand*{\appendixprelim}{}

\newtheoremrep{thm}{Theorem}[section]
\newtheoremrep{cor}[thm]{Corollary}
\newtheoremrep{lemma}[thm]{Lemma}
\newtheoremrep{prop}[thm]{Proposition}

\newtheorem{definition}[thm]{Definition}
\newtheorem{assumption}[thm]{Assumption}
\newtheorem{example}[thm]{Example}

\crefname{thm}{theorem}{theorems}
\crefname{assumption}{assumption}{assumptions}
\crefname{cor}{corollary}{corollaries}
\crefname{prop}{proposition}{propositions}
\crefname{lemma}{lemma}{lemmas}

\usepackage{mdframed}
\newmdtheoremenv{algo}{Algorithm}
\newmdtheoremenv{procedure}{Decision process}

\theoremstyle{remark}
\newtheorem{remark}[thm]{Remark}

\newlist{assnum}{enumerate}{1} %
\setlist[assnum]{label=(\roman*), ref=\theassumption(\roman*)}
\crefalias{assnumi}{assumption} 

\newlist{lemnum}{enumerate}{1} %
\setlist[lemnum]{label=(\roman*), ref=\thelemma(\roman*)}
\crefalias{lemnumi}{lemma} 

\newlist{thmnum}{enumerate}{1} %
\setlist[thmnum]{label=(\roman*), ref=\thethm(\roman*)}
\crefalias{thmnumi}{thm} 

\newlist{cornum}{enumerate}{1} %
\setlist[cornum]{label=(\roman*), ref=\thecor(\roman*)}
\crefalias{cornumi}{cor} 

\newlist{definitionnum}{enumerate}{1} %
\setlist[definitionnum]{label=(\roman*), ref=\thedefinition(\roman*)}
\crefalias{definitionnumi}{definition} 

\newcommand\numberthis{\addtocounter{equation}{1}\tag{\theequation}}		%
\newcommand\dueto[1]{\text{\footnotesize#1}}

\usepackage{siunitx}
\usepackage{dsfont}

\newcommand{\z}{\bar{z}}
\newcommand{\R}{\mathbb{R}}
\newcommand{\dom}{\text{dom}}

\newcommand{\proj}{\operatorname{proj}}

\newcommand{\set}[1]{\left\{#1\right\}}

\newenvironment{defbox}{}{}
\newenvironment{asmbox}{}{}
\newenvironment{thmbox}{}{}

\usepackage{braket}
\newcommand{\id}{\operatorname{id}}
\newcommand{\zer}{\operatorname{zer}}
\newcommand{\fix}{\operatorname{fix}}
\newcommand{\graph}{\operatorname{grph}}
\newcommand{\dist}{\operatorname{dist}}

\newcommandx{\QC}[2][1={},2={}]{\ifstrempty{#1}{Q#2}{Q_{#1}\ifstrempty{#2}{}{(#2)}}}
\newcommandx{\PC}[2][1={},2={}]{\ifstrempty{#1}{P#2}{P_{#1}\ifstrempty{#2}{}{(#2)}}}
\newcommandx{\HC}[2][1={},2={}]{\ifstrempty{#1}{H#2}{H_{#1}\ifstrempty{#2}{}{(#2)}}}
\newcommandx{\MC}[2][1={},2={}]{\ifstrempty{#1}{M#2}{M\ifstrempty{#2}{}{(#2)}}}

\newcommandx{\EF}[2][1={k},2={}]{\mathbb E\ifstrempty{#1}{}{_{#1}}#2}

\usepackage{algorithm}
\usepackage[noend]{algpseudocode}
\newcommand{\algfont}[1]{\textbf{#1}}

\usepackage{quiver}

\usepackage{adjustbox}

\usepackage{array,multirow,hhline,graphicx}
\usepackage{colortbl}
\newcommand{\STAB}[1]{\begin{tabular}{@{}c@{}}#1\end{tabular}}
\newcommand{\cellhl}{\cellcolor{SkyBlue!40}}

\usepackage{pifont}
\newcommand{\xmark}{\ding{55}}
\newcommand*\ruleline[1]{\par\noindent\raisebox{.8ex}{\makebox[\linewidth]{\hrulefill\hspace{1ex}\raisebox{-.8ex}{#1}\hspace{1ex}\hrulefill}}}

\usepackage{wrapfig}

\usepackage{booktabs}
\usepackage{graphicx}

\usepackage[final]{neurips_2023}

\usepackage[utf8]{inputenc} %
\usepackage[T1]{fontenc}    %
\usepackage{hyperref}
\usepackage{url}            %
\usepackage{booktabs}       %
\usepackage{amsfonts}       %
\usepackage{nicefrac}       %
\usepackage{microtype}      %

\title{Stable Nonconvex-Nonconcave Training \\ via Linear Interpolation}

\author{%
  Thomas Pethick \\
  EPFL (LIONS) \\
  \texttt{thomas.pethick@epfl.ch}
  \And 
  Wanyun Xie \\
  EPFL (LIONS) \\
  \texttt{wanyun.xie@epfl.ch}
  \And
  Volkan Cevher \\
  EPFL (LIONS) \\
  \texttt{volkan.cevher@epfl.ch}
}

\begin{document}

\maketitle
\begin{abstract}
This paper presents a theoretical analysis of linear interpolation as a principled method for stabilizing (large-scale) neural network training. We argue that instabilities in the optimization process are often caused by the nonmonotonicity of the loss landscape and show how linear interpolation can help by leveraging the theory of nonexpansive operators. We construct a new optimization scheme called relaxed approximate proximal point (RAPP), which is the first explicit method without anchoring to achieve last iterate convergence rates for $\rho$-comonotone problems while only requiring $\rho > -\tfrac{1}{2L}$.\footnote{The range of RAPP can be extended to $\rho > -\tfrac{1}{L}$ by using the weaker notion of \emph{conic} nonexpansiveness as first exploited in \citet{alacaoglu2024extending} to derive first-order complexity results (cf. \Cref{app:conic}).} The construction extends to constrained and regularized settings. By replacing the inner optimizer in RAPP we rediscover the family of Lookahead algorithms for which we establish convergence in cohypomonotone problems even when the base optimizer is taken to be gradient descent ascent. The range of cohypomonotone problems in which Lookahead converges is further expanded by exploiting that Lookahead inherits the properties of the base optimizer. We corroborate the results with experiments on generative adversarial networks which demonstrates the benefits of the linear interpolation present in both RAPP and Lookahead.

\end{abstract}

\etocdepthtag.toc{mtchapter}
\etocsettagdepth{mtchapter}{subsection}
\etocsettagdepth{mtappendix}{none}

\begin{toappendix}
\section{Additional related work}\label{app:relatework}
\paragraph{Stochastic feedback}
  There are several ways in which a stochastic variant of PP can be devised.
  \emph{Incremental proximal methods} were pioneered for convex minimization in \citep{bertsekas2011incremental}, which uses an implicit update conditioned on the current randomness. 
  Related approaches include \citet{patrascu2017nonasymptotic,bianchi2015convergence,patrascu2021stochastic,toulis2016towards}.
  Alternatively, \citep{toulis2015proximal} assumes noisy access to the \emph{full batch} implicit update in what they call the \emph{proximal Robbins-Monro precedure}, which is similar to the approach taken in \citet{bravo2022stochastic} concerning Krasnoselskii-Mann iterations.
  \Citet{toulis2015proximal} explicitly approximate the implicit update in the \emph{proximal stochastic fixed-point} algorithm which is closely related to the approximation in \Cref{sec:onestep}. %
  In the cohypomonotone case it is common to rely on increasing batchsizes (see e.g. \citep[Thm. 4.5]{diakonikolas2021efficient} and \citep[Thm. 6.1]{lee2021fast}) similarly to \Cref{cor:inexactResolvent:stoc}. %
  Very recently, \citep{anonymous2023solving} showed that convergence in stochastic weak MVI (and thus cohypomonotone problems) is possible for an extragradient-type scheme if the Lipschitz conditions are further tightened to a mean-squared smoothness assumption on the stochastic oracles.

\paragraph{Halpern-type}
Halpern iteration introduced in \citet{halpern1967fixed}, in contrast with \ref{eq:IKM}, linearly interpolates with the initial point using a time-varying stepsize, i.e. $z^{k+1} = (1-\lambda_k)z^0 - \lambda_k Tz^k$.
A $\mathcal O(1/k^2)$ convergence rate for the squared fixed point residual was shown in \citet{lieder2021convergence} for nonexpansive operators.
By directly approximating the Halpern iteration, an explicit scheme for monotone problems was later proposed in \cite{diakonikolas2020halpern}, but it suffered a logarithmic factor in the rate.
The logarithmic factor was later removed by means of an extragradient variant \citep{yoon2021accelerated}.
The scheme was extended to unconstrained cohypomonotone problems in \citet{lee2021fast} and subsequently the constrained case in \citet{cai2022accelerated} while only requiring a single projection.

For a detailed discussion on how Halpern-type methods are not 1-SCLI algorithms see \citet[Appendix E.2]{yoon2021accelerated}, which specifically addresses the anchored extragradient method.
The extragradient method, on the other hand, can be written as an 1-SCLI algorithm (cf. \citet[Def. 5]{golowich2020last} and the subsequent discussion).
This argument extends to the multistep extragradient construction used in \RAPP.

\end{toappendix}

\begin{toappendix}
\section{Preliminaries}\label{app:preliminaries}

\begin{figure*}[t]
\begin{adjustbox}{center}
\begin{tabular}{c}
$\begin{tikzcd}
	{} & \text{Lookahead (\Cref{thm:LA:nonexpansive})} & \text{\ref{eq:lookahead} (\Cref{thm:LA:cocoercive})} & {\tfrac 12 \operatorname{GD} + \tfrac 12 \operatorname{EG+}} \text{ (\Cref{thm:LA:k2})} \\
	\text{\ref{eq:IKM}} & {\text{\ref{eq:EG+} (\Cref{thm:EG+})}} & {\text{\ref{eq:LA-EG} (\Cref{thm:FBF})}} & \text{\ref{eq:EG+} (\Cref{thm:EG+})} \\
	& \text{\RAPP (\Cref{cor:inexactResolvent})} & {\text{\ref{eq:LA-CEG+} (\Cref{thm:LA:CEG+})}} &
	\arrow["{\widetilde{T}_k=\operatorname{EG}}"', from=2-1, to=2-2]
	\arrow["{\text{$\widetilde{T}_k$ is a solver for \eqref{eq:ppm0}}}"', from=2-1, to=3-2]
	\arrow["\text{$\widetilde{T}_k$ is an iterative solver}", from=2-1, to=1-2]
	\arrow[from=1-2, to=3-3]
	\arrow[from=1-2, to=2-3]
	\arrow[from=1-2, to=1-3]
	\arrow["{\tau=2}", from=3-2, to=2-2]
	\arrow["{\tau=1}", from=2-3, to=2-4]
	\arrow["{\tau=2}", from=1-3, to=1-4]
\end{tikzcd}$
\end{tabular}
\end{adjustbox}
\caption{Overview of results and relationship between methods.}
\label{fig:overview}
\end{figure*}

The distance from $z \in \R^d$ to a set $\mathcal Z \subseteq \R^d$ is defined as $\dist(z,\mathcal Z) := \min_{z' \in \mathcal Z} \|z-z'\|$.
The normal cone is defined as ${\mathcal N}_{\mathcal Z}(z) := \set{v \mid \braket{v, z' - z} \leq 0 \quad \forall z' \in \mathcal Z}$
and the projection as $\boldsymbol{\Pi}_{\mathcal{Z}}(z) := \min_{w \in \mathcal{Z}} \|z-w\|^2$.
We will denote the natural filtration up to iteration $k$ as $\mathcal F_k$ and use $\EF[k][[\cdot]]=\mathbb E[\cdot \mid \mathcal F_{k}]$.

We restate here some common definitions from monotone and nonexpansive operator for convenience (for further details see
\citet{Bauschke2017Convex}).
An operator $A:\R^d\rightrightarrows\R^n$ maps each point $z\in\R^d$ to a subset $Az \subseteq \R^n$, where the notation $A(z)$ and $Az$ will be used interchangably. 
We denote the domain of $A$ by 
$\dom A:=\{z\in\R^d\mid Az\neq\emptyset\},$
its graph by 
$\graph A:=\{(z,v)\in\R^d\times \R^n\mid v\in Az\}.$
The inverse of $A$ is defined through its graph, $\graph A^{-1}:=\{(v,z)\mid (z,v)\in\graph A\}$ and
the set of its zeros by $\zer A:=\{z\in\R^d \mid 0\in Az\}$.
The set of fixed points is defined as $\fix T:=\{z\in\R^d \mid z\in Tz\}$ for the operator $T\colon \R^d \rightrightarrows \R^d$.

\begin{definition}
A single-valued operator $T\colon \R^d \rightarrow \R^d$ is said to be 
\begin{definitionnum}
\item nonexpansive if $\|Tz - Tz'\| \leq \|z-z'\| \quad \forall z,z' \in \R^d$.
\item quasi-nonexpansive if $\|Tz - z^\star\| \leq \|z-z^\star\| \quad \forall z \in \R^d\text{ and } \forall z^\star \in \fix T$.
\item firmly nonexpansive if $\|Tz - Tz'\|^2 \leq \|z-z'\|^2 - \|(z-z')-(Tz - Tz')\|^2 \quad \forall z,z' \in \R^d$.
\end{definitionnum}
\end{definition}
The resolvent operator  $J_A:=(\id + A)^{-1}$ is firmly nonexpansive (with $\dom J_A= \R^d$) iff $A$ is maximally monotone.
\begin{definition}[(co)monotonicity \cite{bauschke2021generalized}] 
An operator $A:\R^d\rightrightarrows\R^d$ is called monotone if,
 \[
 \langle v-v^\prime, z-z^\prime\rangle \geq 0 \quad \forall (z,v),(z^\prime,v^\prime)\in\graph A,
 \]
 and the operator $A$ is called $\rho$-comonotone (also referred to as $|\rho|$-cohypomonotonicity when $\rho<0$) if
\[
 \langle v-v^\prime, z-z^\prime\rangle \geq \rho\|v-v^\prime\|^2 \quad \forall (z,v),(z^\prime,v^\prime)\in\graph A.
 \]
The operator $A$ is \emph{maximally} (co)monotone if no other (co)monotone operator $B$ exists for which $\graph A \subset \graph B$.
\end{definition}
\begin{definition}[Lipschitz continuity and cocoercivity]
    Let $\mathcal D\subseteq \R^d$ be a nonempty set. A single-valued operator $A:\mathcal D\to \R^n$ is said to be $L$-Lipschitz continuous if for any $z,z^\prime\in \mathcal D$
	\[
	\|Az - Az^\prime\| \leq L\|z-z^\prime\|,
	\]
	and $\beta$-cocoercive if 
	\[
		\langle z-z^\prime, Az - Az^\prime \rangle \geq \beta\|Az - Az^\prime\|^2. 
	\]
\end{definition}

The forward step $H=\id - \gamma F$ is $\nicefrac 12$-cocoercive when $F$ is Lipschitz continuity and $\gamma$ is sufficiently small.

\begin{lemma}[{\citet[Lm. A.3(i)]{pethick2022escaping}}]\label{lm:Main:H:properties}
Suppose \Cref{ass:F:Lips} holds and $\gamma \leq \nicefrac 1L$. Then, 
the mapping \(\HC[] = \id - \gamma F\) is \(\nicefrac12\)-cocoercive for all $u \in \R^d$. Specifically,
\begin{equation}
\langle \HC[][z']-\HC[][z],z'-z\rangle \geq 
  \tfrac12\|\HC[][z']-\HC[][z]\|^2 
  + \tfrac12 (1-\gamma^2L^2) \|z'-z\|^2
  \quad \forall z,z'\in\R^d.
\end{equation}
\end{lemma}
\begin{proof}
By expanding,
\begin{align*}
    \HC[][z] - \HC[][z'] 
        {}={}
    (z -z') - \gamma(Fz - Fz').
    \numberthis\label{eq:Main:Hdif:1}
\end{align*}
Using \eqref{eq:Main:Hdif:1} we get,  
\begin{align*}
    \langle \HC[][z']-\HC[][z],z'-z\rangle
        {}={}&
    \langle \HC[][z']-\HC[][z],\HC[][z']-\HC[][z] - \gamma(Fz - Fz')\rangle
    \\
    \dueto{\eqref{eq:Main:Hdif:1}}    
        {}={}&
    \tfrac12\|\HC[][z']-\HC[][z]\|^2
        {}+{}
    \tfrac12 \|z'-z\|^2 
        {}-{}
    \tfrac{\gamma^2}{2} \|Fz - Fz'\|^2
    \\
    \dueto{(\Cref{ass:F:Lips})}
        {}\geq{}&
    \tfrac12\|\HC[][z']-\HC[][z]\|^2 + \tfrac12 (1-\gamma^2L^2) \|z'-z\|^2
    \numberthis\label{eq:Main:AFBA:cocoercive}
\end{align*}
This completes the proof.
\end{proof}

\subsection{Relationship between weak Minty variational inequilities and cohypomonotonicity}\label{app:weakMVI-comonotone}

Let $F\colon \R^d\rightarrow\R^d$ be a single-valued operator.
In the unconstrained case, the weak Minty variational inequality (MVI) with parameter $\rho \in (-\nicefrac{1}{2L},\infty)$ is defined as
\begin{equation}
\braket{Fz,z-z^\star} \geq \rho \|Fz\|^2
\quad
\forall z\in \R^d, 
\forall z^\star \in \zer F.
\end{equation}
For $\rho < 0$, this condition allows the operator $-Fz$ to point away from the solution set as illustrated in \Cref{fig:weakMVI}.

Notice that since $z^\star \in \zer F$ we could equivalently write
\begin{equation}
\braket{Fz-Fz^\star,z-z^\star} \geq \rho \|Fz-Fz^\star\|^2.
\quad \forall z\in \R^d
\end{equation}
In contrast, $\rho$-comonotonicity of $F$ states that the above condition should hold for all pairs of point in the domain, i.e.
\begin{equation*}
\langle Fz - Fz',z - z'\rangle \geq \rho\|Fz - Fz'\|^2 \quad \forall z,z' \in \R^d.
\end{equation*}
For $\rho < 0$, $\rho$-comonotonicity is also referred to as $|\rho|$-cohypomonotonicity.
We say that the weak MVI is a \emph{star-variant} of comonotonicity.
This is analogue to the relationship between convexity and star-convexity.

For simplicity we state all results in terms of comonotonicity.
However, note that \emph{almost all results in this paper trivially extends to the more relaxed notion of weak MVI}.
The only exception is the last iterate rates in \Cref{thm:inexactResolvent:last} which relies on cohypomonotonicity to prove monotonic decrease through \Cref{eq:lastiter}.

\end{toappendix}

    \vspace{-1.0em}
    \section{Introduction}\label{sec:introduction}
    \vspace{-1.0em}
     Stability is a major concern when training large scale models.
 In particular, generative adversarial networks (GANs) are known to be notoriously difficult to train.
 To stabilize training, the Lookahead algorithm of \citet{zhang2019lookahead} was recently proposed for GANs \cite{chavdarova2020taming} which linearly interpolates with a slow moving iterate.
 The mechanism has enjoyed superior empirical performance in both minimization and minimax problems, but it largely remains a heuristic with little theoretical motivation.
 
 One major obstacle for providing a theoretical treatment, is in capturing the (fuzzy) notion of stability.
 Loosely speaking, a training dynamics is referred to as \emph{unstable} in practice when the iterates either cycle indefinitely or (eventually) diverge---as has been observed for the Adam optimizer (see e.g. \citet[Fig. 12]{gidel2018variational} and \citet[Fig. 6]{chavdarova2020taming} respectively).
 Conversely, a \emph{stable} dynamics has some bias towards stationary points.
 The notion of stability so far (e.g. in \citet[Thm. 2-3]{chavdarova2020taming}) is based on the spectral radius and thus inherently \emph{local}.

\begin{table}[t]
\caption{Overview of last iterate results with our contribution highlighted in \colorbox{SkyBlue!40}{blue}.
Prior to this work there existed no rates for explicit schemes without anchoring handling $\rho$-comonotone problems with $\rho \in (-\nicefrac{1}{2L}, \infty)$ and no global convergence guarantees for Lookahead beyond bilinear games.}
\label{tbl:overview}
\centering
\footnotesize%
\begin{adjustbox}{center}
\bgroup
\def\arraystretch{1.2}
\begin{tabular}{|cc|ccccc|}
\hline
& Method & Setting & $\rho$ & Handles constraints & $\rho$-independent rates & Reference \\
\hline\hline
\multirow{3}{*}{\STAB{\rotatebox[origin=c]{90}{\scriptsize \hspace{10pt} Implicit}}}
	& PP & Comonotone & $(-\nicefrac{1}{2L}, \infty)$ & $\checkmark$ & \xmark   & \cite[Thm. 3.1]{gorbunov2022convergence} \\ %
	& \ref{eq:RPP} & \cellhl Comonotone & \cellhl $(-\nicefrac{1}{2L}, \infty)$ & \cellhl $\checkmark$ & \cellhl $\checkmark$  & \cellhl \Cref{thm:inexactResolvent:last} \\ %
\hline\hline
\multirow{3}{*}{\STAB{\rotatebox[origin=c]{90}{\scriptsize \ Extrapolate}}}
	& EG & Comonotone \& Lips. & $(-\nicefrac{1}{8L}, \infty)$ & \xmark & \xmark & \cite[Thm. 4.1]{gorbunov2022convergence} \\ %
	& \ref{eq:EG+} & Comonotone \& Lips. & \multicolumn{4}{c|}{\ruleline{Unknown rates}} \\
	& \cellhl \RAPP & \cellhl Comonotone \& Lips. & \cellhl $(-\nicefrac{1}{2L}, \infty)$ & \cellhl $\checkmark$ & \cellhl $\checkmark$ & \cellhl \Cref{cor:inexactResolvent} \\ %
\hline\hline
\multirow{5}{*}{\STAB{\rotatebox[origin=c]{90}{\scriptsize Lookahead}}}
	& \multirow{3}{*}{\ref{eq:lookahead}}
		& Local & - & \xmark & - & \cite[Thm. 2]{chavdarova2020taming}\\
	  & & Bilinear & - & \xmark & - & \cite[Cor. 7]{ha2022convergence} \\
		&  & \cellhl Comonotone \& Lips. & \cellhl $(-\nicefrac{1}{3\sqrt{3}L}, \infty)$  & \cellhl \xmark & \cellhl - & \cellhl \Cref{thm:LA:k2}\\
	\hhline{|~|------|}
	& \multirow{2}{*}{\ref{eq:LA-EG}}
 		& Bilinear & - & \xmark & - & \cite[Cor. 8]{ha2022convergence} \\
		& & \cellhl Monotone \& Lips. & \cellhl - & \cellhl $\checkmark$ & \cellhl - & \cellhl \Cref{thm:FBF} \\
	\hhline{|~|------|}
	& \cellhl \ref{eq:LA-CEG+} & \cellhl Comonotone \& Lips. & \cellhl $(-\nicefrac{1}{2L}, \infty)$ & \cellhl $\checkmark$ & \cellhl - & \cellhl \Cref{thm:LA:CEG+} \\
\hline
\end{tabular}
\egroup
\end{adjustbox}
\vspace{-7pt}
\end{table}

In this work, we are interested in establishing \emph{global} convergence properties, in which case some structural assumptions are needed.
One (nonmonotone) structure that lends itself well to the study of stability is that of cohypomonotonicity studied in \citet{combettes2004proximal,diakonikolas2021efficient}, since even the extragradient method has been shown to cycle and diverge in this problem class (see \citet[Fig. 1]{pethick2022escaping} and \citet[Fig. 2]{anonymous2023solving} respectively).
We provide a geometric intuition behind these difficulties in \Cref{fig:weakMVI}.
Biasing the optimization schemes towards stationary points becomes a central concern and we demonstrate in \Cref{fig:forsaken} that Lookahead can indeed converge for such nonmonotone problems.

A principled approach to cohypomonotone problems is the extragradient+ algorithm \eqref{eq:EG+} proposed by \citet{diakonikolas2021efficient}.
However, the only known rates are on the best iterate, which can be problematic to pick in practice. 
It is unclear whether \emph{last} iterate rates for \ref{eq:EG+} are possible even in the monotone case (see discussion prior to Thm. 3.3 in \citet{gorbunov2022extragradient}).
For this reason, the community has instead resorted to showing last iterate of extragradient (EG) method of \citet{korpelevich1977extragradient}, despite originally being developed for the monotone case.
Maybe not surprisingly, EG only enjoys a last iterate guarantee under mild form of cohypomonotonicity and have so far only been studied in the unconstrained case \citep{luoextragradient,gorbunov2022convergence}.
Recently, last iterate rate were established for the same (tight) range of cohypomonotone problems for which \ref{eq:EG+} has best iterate guarantees.
However, the analyzed scheme is \emph{implicit} and the complexity blows up with increasing cohypomonotonicity \citep{gorbunov2022convergence}.
This leaves the questions: \emph{Can an explicit scheme enjoy last iterate rates for the same range of cohypomonotone problems? Can the rate be agnostic to the degree of cohypomonotonicity?}
We answer both in the affirmative.

This work focuses on 1-SCLI schemes \citep{arjevani2015lower,golowich2020last}, whose update rule only depends on the previous iterate in a time-invariant fashion. %
Another approach to establishing last iterate is Halpern-type methods
with an explicit scheme developed in \citet{lee2021fast} for cohypomonotone problems
and later extended to the constrained case in \citet{cai2022accelerated}
(cf. \Cref{app:relatework}).
\looseness=-1

As will become clear, a principled mechanism behind convergence in this nonmonotone class is the linear interpolation also used in Lookahead.
This iterative interpolation is more broadly referred to as the Krasnosel'skiĭ-Mann (KM) iteration 
in the theory of nonexpansive operators.
We show that the extragradient+ algorithm \eqref{eq:EG+} of \citet{diakonikolas2021efficient}, our proposed relaxed approximate proximal point method (\RAPP), and Lookahead based algorithms are all instances of the (inexact) KM iteration
and provide simple proofs of these schemes in the cohypomonotone case.

More concretely we make the following contributions:
\vspace{-3pt}
\begin{enumerate}[leftmargin=1.5em]
  \item We prove global convergence rates for the last iterate of our proposed algorithm \RAPP which additionally handles constrained and regularized settings. This makes \RAPP the first 1-SCLI scheme to have non-asymptotic guarantees for $\rho$-comonotone problems while only requiring $\rho > -\nicefrac{1}{2L}$.
	As a byproduct we obtain a last iterate convergence rate for an implicit scheme that is \emph{independent} of the degree of cohypomonotonicity.
	The last iterate rates are established by showing monotonic decrease of the operator norm--something which is not possible for \ref{eq:EG+}.
	This contrast is maybe surprising, since \RAPP can be viewed as an extension of \ref{eq:EG+}, which simply takes multiple extrapolation steps.
  \item By replacing the inner optimization routine in \RAPP with gradient descent ascent (GDA) and extragradient (EG)  we rediscover the Lookahead algorithms considered in \cite{chavdarova2020taming}.
  We obtain guarantees for the Lookahead variants by deriving nonexpansive properties of the base optimizers.
	By casting Lookahead as a KM iteration we find that the optimal interpolation constant is $\lambda=0.5$. 
	This choice corresponds to the default value used in practice for both minimization and minimax---thus providing theoretical motivation for the parameter value.
  \item For $\tau=2$ inner iterations we observe that \ref{eq:lookahead} reduces to a linear interpolation between GDA and \ref{eq:EG+} which allows us to obtain global convergence in $\rho$-comonotone problems when $\rho > -\nicefrac{1}{3\sqrt{3}L}$. %
  However, for $\tau$ large, we provide a counterexample showing that \ref{eq:lookahead} cannot be guaranteed to converge.
  This leads us to instead propose \ref{eq:LA-CEG+} which corrects the inner optimization to guarantee global convergence for $\rho$-comonotone problems when $\rho > -\nicefrac{1}{2L}$.
  \item We test the methods on a suite of synthetic examples and GAN training where we confirm the stabilizing effect. 
	Interestingly, \RAPP seems to provide a similar benefit as Lookahead, which suggest that linear interpolation could play a key role also experimentally. %
\end{enumerate}
An overview of the theoretical results is provided in \Cref{tbl:overview} and \Cref{fig:overview}§\ref{app:preliminaries}.

    \section{Related work}\label{sec:relatedwork}
    
\paragraph{Lookahead}
The Lookahead algorithm was first introduced for minimization in \citet{zhang2019lookahead}.
In the context of Federated Averaging in federated learning \citep{mcmahan2017communication} and the Reptile algorithm in meta-learning \citep{nichol2018first}, the method can be seen as a single worker and single task instance respectively.
Analysis for Lookahead was carried out for nonconvex minimization \citep{wang2020lookahead,zhou2021towards} and a nested variant proposed in \citep{pushkin2021multilayer}.
\Citet{chavdarova2020taming} popularized the Lookahead algorithm for minimax training by showing state-of-the-art performance on image generation tasks.
Apart from the original local convergence analysis in \citet{chavdarova2020taming} and the bilinear case treated in \cite{ha2022convergence} we are not aware of any convergence analysis for Lookahead for minimax problems and beyond.

\paragraph{Cohypomonotone}
Cohypomontone problems were first studied in \citet{iusem2003inexact, combettes2004proximal} for proximal point methods and later expanded on in greater detail in \citet{bauschke2021generalized}.
The condition was relaxed to the star-variant referred to as the weak Minty variational inequality (MVI) in \citet{diakonikolas2021efficient} and the extragradient+ algorithm \eqref{eq:EG+} was analyzed.
The analysis of \ref{eq:EG+} was later tightened and extended to the constrained case in \citet{pethick2022escaping}.

\paragraph{Proximal point}
The proximal point method (PP) has a long history.
For maximally monotone operators (and thus convex-concave minimax problems) convergence of PP follows from \citet{opial1967weak}.
The first convergence analysis of \emph{inexact} PP dates back to \citet{rockafellar1976monotone,brezis1978produits}.
It was later shown that convergence also holds for the \emph{relaxed} inexact PP as defined in \eqref{eq:inexactResolvent} \citep{eckstein1992douglas}.
In recent times, PP has gained renewed interest due to its success for certain nonmonotone structures.
Inexact PP was studied for cohypomontone problems in \citet{iusem2003inexact}.
Asymptotic convergence was established of the relaxed inexact PP for a sum of cohypomonotone operators in \citet{combettes2004proximal}, and later considered in \citet{grimmer2022landscape} without inexactness.
Last iterate rates were established for PP in $\rho$-comonotone problems (with a dependency on $\rho$) \citep{gorbunov2022convergence}.
Explicit approximations of PP through a contractive map was used for convex-concave minimax problems in \citet{https://doi.org/10.48550/arxiv.2301.03931} and was the original motivation for MirrorProx of \citet{nemirovski2004prox}.
See \Cref{app:relatework} for additional references in the stochastic setting.

    \section{Setup}\label{sec:setup}
    \begin{wrapfigure}{r}{6cm}
\vspace{-35pt}
\centering
\includegraphics[width=0.42\textwidth,trim={2.5cm 9cm 5cm 4cm},clip]{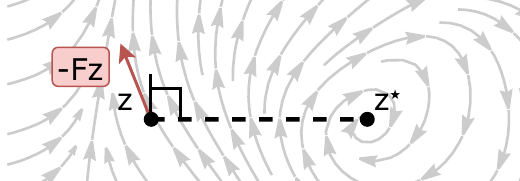}
\caption{Consider $\min _{x \in \mathcal X} \max _{y \in \mathcal Y}\phi(z)$ with $z=(x,y)$. 
As opposed to convex-concave minimax problems, the cohypomonotone condition allows the gradients $Fz= (\nabla_x \phi(z), -\nabla_y \phi(z))$ to point away from the solutions (see \Cref{app:weakMVI-comonotone} for the relationship between cohypomonotonicity and the weak MVI).
This can lead to instability issues for standard algorithms such as the Adam optimizer.}
\label{fig:weakMVI}
\end{wrapfigure}

We are interested in finding a zero of an operator $S:\R^d\rightrightarrows\R^d$ which decomposes into a Lipschitz continuous (but possibly nonmonotone) operator $F$ and a maximally monotone operator $A$, i.e. find $z \in \R^d$ such that,
\begin{equation}\label{eq:StrucIncl}
	0\in Sz := Az + Fz.
\end{equation}
Most relevant in the context of GAN training is that \eqref{eq:StrucIncl} includes constrained minimax problems.
\begin{example}
Consider the following minimax problem
\begin{equation}
\label{eq:minimax}
\min _{x \in \mathcal X} \max _{y \in \mathcal Y}\phi(x, y).
\end{equation}
The problem can be recast as the inclusion problem \eqref{eq:StrucIncl} by defining the joint iterates $z=(x,y)$, the stacked gradients $Fz = (\nabla_x \phi(x,y), -\nabla_y \phi(x,y))$, and $A=(\mathcal{N}_{\mathcal X}, \mathcal{N}_{\mathcal Y})$ where $\mathcal{N}$ denotes the normal cone.
As will become clear (cf. \Cref{alg:inexactResolvent}), $A$ will only be accessed through the resolvent $J_{\gamma A}:=(\id+\gamma A)^{-1}$ which reduces to the proximal operator. More specifically $J_{\gamma A}(z)=(\proj_{\mathcal X}(x), \proj_{\mathcal Y}(y))$.
\end{example}
We will rely on the following assumptions (see \Cref{app:preliminaries} for any missing definitions).
\begin{asmbox}
\begin{assumption}\label{ass:S}
In problem \eqref{eq:StrucIncl}, 
\begin{assnum}
\itemsep0em 
\item\label{ass:A:maxmon}
	The operator $A:\R^d\rightrightarrows\R^d$ is maximally monotone. 
\item \label{ass:F:Lips} The operator $F: \R^d\rightarrow\R^d$ is $L$-Lipschitz, i.e. for some $L \in [0,\infty)$,
    \begin{equation*}
    \|Fz - Fz'\|
    \leq
    L \|z-z'\|
    \quad \forall z,z' \in \R^d.
    \end{equation*}
\item\label{ass:F:cohypo} The operator $S:=F+A$ is $\rho$-comonotone for some $\rho \in (-\nicefrac{1}{2L},\infty)$, i.e.
\begin{equation*}
\langle v - v',z - z'\rangle \geq \rho\|v - v'\|^2 \quad \forall (v,z),(v',z') \in \graph S.
\end{equation*}
\end{assnum}
\end{assumption}
\end{asmbox}
\begin{remark}
\Cref{ass:F:cohypo} is also known as $|\rho|$-cohypomonotonicity when $\rho < 0$, which allows for increasing nonmonotonicity as $|\rho|$ grows.
See \Cref{app:weakMVI-comonotone} for the relationship with weak MVI.
\end{remark}
When only stochastic feedback $\hat F_\sigma(\cdot,\xi)$ is available we make the following classical assumptions.%
\begin{asmbox}
\begin{assumption}\label{ass:stoch}
For the operator $\hat F_\sigma(\cdot,\xi):\R^d\rightarrow\R^d$ the following holds.
\begin{assnum}
\itemsep0em 
\item \label{ass:stoch:unbiased}
    Unbiased:
    \(%
        \mathbb{E}_{\xi}\big[\hat F_\sigma(z,\xi)\big] = Fz \quad \forall z \in \R^d
    \).%
\item  \label{ass:stoch:boundedvar}
    Bounded variance:
    \(%
        \mathbb{E}_{\xi}\big[\|\hat F_\sigma(z,\xi)- Fz\|^2\big] \leq \sigma^2  \quad \forall z,z' \in \R^d.
    \)%
\end{assnum}
\end{assumption}
\end{asmbox}

\begin{toappendix}
\section{Proofs for \Cref{sec:ikm} (Inexact Krasnosel'skiĭ-Mann iterations)}
\end{toappendix}
    \section{Inexact Krasnosel'skiĭ-Mann iterations}\label{sec:ikm}
    \nosectionappendix
    
The main work horse we will rely on is
the inexact Krasnosel'skiĭ-Mann (IKM) iteration from monotone operators (also known as the \emph{averaged} iteration),
which acts on an operator $T:  \mathbb R^d \rightarrow \mathbb R^d$ with inexact feedback,
\begin{equation}\label{eq:IKM}
\tag{IKM}
z^{k+1} = (1 - \lambda) z^k + \lambda (Tz^k + e^k),
\end{equation}
where $\lambda \in (0,1)$ and $e^k$ is a random variable with dependency on all variables up until (and including) $k$. 
The operator $\widetilde{T}_k: z \mapsto Tz + e^k$ can crucially be an iterative optimization scheme in itself. 
This is important, since we can obtain \RAPP, \ref{eq:lookahead} and \ref{eq:LA-CEG+} by plugging in different optimization routines.
In fact, \RAPP is derived by taking $\widetilde{T}_k$ to be a (contractive) fixed point iteration in itself, which approximates the resolvent.

We note that also the extragradient+ \eqref{eq:EG+} method of \citet{diakonikolas2021efficient}, which converges for cohypomonotone and Lipschitz problems, can be seen as a Krasnosel'skiĭ-Mann iteration on an extragradient step
\begin{equation}
\label{eq:EG+}
\tag{EG+}
\begin{split}
\operatorname{EG}(z) &= z - \gamma F(z - \gamma Fz) \\
z^{k+1} &= (1-\lambda)z^k + \lambda \operatorname{EG}(z^k)
\end{split}
\end{equation}
where $\lambda\in(0,1)$.
We provide a proof of \ref{eq:EG+} in \Cref{thm:EG+} which extends to the constrained case using the construction from \citet{pethick2022escaping} but through a simpler argument under fixed stepsize. %

Essentially, the \ref{eq:IKM} iteration leads to a conservative update that stabilizes the update using the previous iterate.
This is the key mechanism behind showing convergence in the nonmonotone setting known as cohypomonotonicity.
Very generally, it is possible to provide convergence guarantees for \ref{eq:IKM} when the following holds (\Cref{thm:KM:best} is deferred to the appendix due to space limitations).
\begin{defbox}
\begin{definition}\label{def:nonexpansiveness}
An operator $T: \mathbb R^d \rightarrow \mathbb R^d$ is said to be quasi-nonexpansive if
\begin{equation}
\|Tz - z'\| \leq \|z-z'\| \quad \forall z \in \mathbb R^d, \forall z'\in \fix T.
\end{equation}
\end{definition}
\end{defbox}
\begin{remark}\label{rem:nonexpansive}
This notion is crucial to us since the resolvent $J_{B}:=(\id + B)^{-1}$ is (quasi)-nonexpansive if $B$ is $\nicefrac{1}{2}$-cohypomonotone \citep[Prop. 3.9(iii)]{bauschke2021generalized}.
\end{remark}
\begin{toappendix}
The \ref{eq:IKM} iteration is well studied (see \citet{combettes2001quasi}). %
The following result refurbishes sub-results of \citet[Prop. 4.2]{combettes2001quasi} to establish a rate of convergence under potentially stochastic feedback.
\begin{thmbox}
\begin{thm}[Convergence of IKM]\label{thm:KM:best}
Suppose $T: \mathbb R^d \rightarrow \mathbb R^d$ is quasi-nonexpansive. Consider the sequence $(z^k)_{k\in \mathbb N}$ generated by \ref{eq:IKM} with $\lambda \in (0,1)$. Then, for all $z^\star \in \fix T$
\begin{equation}
\frac{1}{K}\sum_{k=0}^{K-1}\EF[][\|Tz^k - z^k\|^2] \leq \frac{\|z^0 - z^\star \|^2 + \sum_{k=0}^{K-1}\varepsilon_k(z^\star)}{\lambda(1-\lambda) K}.
\end{equation}
where $\varepsilon_k(z) = 2\lambda \EF[][[\|e^k\|\|z^{k} - z \|]] + \lambda^2\EF[][[\|e^k\|^2]]$. 
Furthermore, $z^k \rightarrow z^\star$ a.s. as long as $\sum_{k=0}^\infty \varepsilon_k(z^\star) < \infty$.
\end{thm}
\end{thmbox}
\begin{remark}
Notice that the optimal choice of $\lambda$ in the upper bound is $\lambda=0.5$, which is the default used for the Lookahead algorithm in both for minimax problems \citep{chavdarova2020taming} and minimization \citep{zhang2019lookahead} (see \Cref{sec:lookahead} for a treatment of Lookahead).
\end{remark}
\end{toappendix}
\begin{appendixproof}
We will denote the natural filtration up to iteration $k$ as $\mathcal F_k$ and use $\EF[k][[\cdot]]=\mathbb E[\cdot \mid \mathcal F_{k}]$.
Consider one exact step 
\begin{equation}\label{eq:sk}
s^{k} = (1 - \lambda) z^k + \lambda Tz^k
\end{equation}
Then
\begin{align*}
\|s^{k} - z^\star \|^2 
&= (1-\lambda)\|z^k - z^\star \|^2 + \lambda\|Tz^k - z^\star\|^2 - \lambda(1-\lambda)\|Tz^k - z^k\|^2 \\
\dueto{(quasi-nonexpansive)}&\leq 
  (1-\lambda)\|z^k - z^\star \|^2 + \lambda\|z^k - z^\star\|^2 - \lambda(1-\lambda)\|Tz^k - z^k\|^2 \\
&= \|z^k - z^\star \|^2 - \lambda(1-\lambda)\|Tz^k - z^k\|^2
\numberthis \label{eq:exact:decent}
\end{align*}
So
\begin{equation}\label{eq:Texact:decrease}
\|s^{k} - z^\star \| \leq \|z^k - z^\star \|
\end{equation}
By using triangle inequality and the update rule \ref{eq:IKM} we have,
\begin{equation}
\begin{split}
\EF[k][[\|z^{k+1} - z^\star \|^2]]
&\leq \EF[k][[(\|s^{k} - z^\star \| + \lambda\|e^k\|)^2]] \\
& = \|s^{k} - z^\star \|^2 + 2\lambda\EF[k][[\|e^k\|]]\|s^{k} - z^\star \| + \lambda^2\EF[k][[\|e^k\|^2]] \\
\dueto{\eqref{eq:Texact:decrease}}& \leq 
  \|s^{k} - z^\star \|^2 + 2\lambda\EF[k][[\|e^k\|]]\|z^{k} - z^\star \| + \lambda^2\EF[k][[\|e^k\|^2]] \\
\dueto{\eqref{eq:exact:decent}}& \leq 
  \|z^k - z^\star \|^2 - \lambda(1-\lambda)\|Tz^k - z^k\|^2
   + 2\lambda \EF[k][[\|e^k\|]]\|z^{k} - z^\star \| + \lambda^2\EF[k][[\|e^k\|^2]].
\end{split}
\end{equation}
Using law of total expectation and telescoping obtains the claimed rate.
The claimed asymptotic result follows from the Robbins-Siegmund supermartingale theorem \citep[Prop. 2]{bertsekas2011incremental}.
This completes the proof.
\end{appendixproof}

\begin{toappendix}
\subsection{Generalizing to conic nonexpansiveness}\label{app:conic}

Sufficient conditions for convergence of \ref{eq:IKM} can be relaxed beyond nonexpansiveness as originally shown in \citet[Prop. 2.9]{bartz2022conical}.
\Citet{alacaoglu2024extending} used this weaker condition to establish the first first-order complexity results under cohypomonotonicity with only $\rho > -\gamma$ using a similar construction to \RAPP.
In this section we provide a more direct (but equivalent) condition to conic nonexpansiveness and show that also \RAPP enjoys the weaker requirement of $\rho > -\gamma$.

Notice that to establish Fejér monotonicity in \eqref{eq:exact:decent} we do not strictly need (quasi)-nonexpansiveness. 
It suffice to have
\begin{equation}\label{eq:def:quasi-conic}
\|Tz - z^\star\|^2 \leq \|z-z^\star\|^2 + \beta\|Tz-z\|^2 \quad \forall z \in \mathbb R^d, \forall z^\star \in \fix T,
\end{equation}
for some $\beta < 1$, since then
\begin{align*}
\|s^{k} - z^\star \|^2 
&= (1-\lambda)\|z^k - z^\star \|^2 + \lambda\|Tz^k - z^\star\|^2 - \lambda(1-\lambda)\|Tz^k - z^k\|^2 \\
\dueto{\eqref{eq:def:quasi-conic}}&\leq 
  (1-\lambda)\|z^k - z^\star \|^2 + \lambda\|z^k - z^\star\|^2 - \lambda(1-\beta-\lambda)\|Tz^k - z^k\|^2 \\
&= \|z^k - z^\star \|^2 - \lambda(1-\beta-\lambda)\|Tz^k - z^k\|^2.
\numberthis \label{eq:conic:exact:decent}
\end{align*}
It is clear from \eqref{eq:conic:exact:decent} that the only price we pay for using the weaker condition in \eqref{eq:def:conic} is that the eventually convergence rate will suffer due to $\beta$.
In other words, the rest of the proof for \Cref{thm:KM:best} follows albeit with a different constant in front of $\|Tz^k-z^k\|^2$.

It is easy to show that condition \eqref{eq:def:conic} is implied by $\theta$-conic nonexpansiveness introduced in \citet[Def. 3.1]{bauschke2021generalized} when taking $\beta=\tfrac{\theta-1}{\theta}$.
Formally we have the following proposition.

\begin{definition}[{\citet[Def. 3.1]{bauschke2021generalized}}]\label{eq:def:conic:orig}
The operator $T: \R^d \rightarrow \R^d$ is said to be $\theta$-conic nonexpansive if $N=(1-\theta)\id + \theta T$ for $\theta > 0$ is nonexpansive.
\end{definition}
\begin{prop}
An operator $T: \R^d \rightarrow \R^d$ is $\theta$-conic nonexpansive with $\theta > 0$ if and only if $T$ satisfies the following for $\beta = \tfrac{\theta-1}{\theta} < 1$
\begin{equation}\label{eq:def:conic}
\|Tz - Tz'\|^2 \leq \|z-z'\|^2 + \beta\|(\id-T)z-(\id-T)z'\|^2 \quad \forall z,z' \in \mathbb R^d.
\end{equation}
\end{prop}
\begin{remark}
The condition \eqref{eq:def:conic} is more direct than \Cref{eq:def:conic:orig} in the sense that the auxiliary nonexpansive operator $N$ is avoided.
Notice that the condition is simply a generalization of $\theta$-averageness \citep[Def. 2.1(ii)]{bauschke2021generalized} that allows for any $\theta > 0$ instead of only $\theta \in (0,1)$.
\end{remark}
\begin{proof}
Using that $\id-T=\theta(\id - N)$, the condition \eqref{eq:def:conic} is equivalent to
\begin{equation}
\|Nz - Nz'\|^2 \leq \|z-z'\|^2 - (\theta - 1 - \beta\theta)\|(\id-N)z-(\id-N)z'\|^2.
\end{equation}
If $N$ is nonexpansive (i.e. $\theta - 1 - \beta\theta=0$) we have $\beta = \tfrac{\theta - 1}{\theta}$.
Since $\theta > 0$ it holds that $\beta < 1$.
Conversely, if \eqref{eq:def:conic} holds with $\beta = \tfrac{\theta - 1}{\theta} < 1$ then $\theta > 0$ and $N$ is nonexpansive by construction of $\beta$.
This establishes the equivalence.
\end{proof}

From \citet[Prop. 3.7]{bauschke2021generalized} we know that when $\tilde{\rho} > -1$ the resolvent $J_{A}$ is $\theta$-conic nonexpansive with $\theta=\tfrac{1}{2(\tilde{\rho}+1)}$ iff $A:\R^d \rightrightarrows \R^d$ is $\tilde{\rho}$-comonotone.
Consequently, the resolvent $J_{\gamma S}$ approximated in \RAPP satisfies \eqref{eq:def:conic} with $\beta = \tfrac{\theta-1}{\theta}=-(1+\tfrac{2 \rho}{\gamma})<1$ when $S$ is $\rho$-comonotone with $\rho > -\gamma$.
Convergence can thus be established for \RAPP under the weaker conditions of $\rho > -\gamma$.
Last iterate rates for \RAPP similarly follows, since the arguments for establishing monotonic decrease in \Cref{eq:lastiter} still applies, provided we additionally assume $\lambda \leq 2(1+\tfrac{\rho}{\gamma})$.
A last iterate guarantee for a similar scheme was first mentioned in \citet[Rem. B.6]{alacaoglu2024extending}.

Notice that we originally in \eqref{eq:def:quasi-conic} only required \eqref{eq:def:conic} to hold with respect to the solution.
This quasi-variant is first studied in detail in \citet{alacaoglu2024extending} and generalizes convergence for cohypomonotone problems to convergence for weak Minty variational inequalities for the best iterate.
Under cohypomonotonicity they show optimal last iterate rates using Halpern iteration.

\end{toappendix}

\begin{toappendix}
\section{Proofs for \Cref{sec:onestep} (Approximating the resolvent)}
\end{toappendix}
    \section{Approximating the resolvent}\label{sec:onestep}
    \nosectionappendix

As apparent from \Cref{rem:nonexpansive}, the \ref{eq:IKM} iteration would provide convergence to a zero of the cohypomonotone operator $S$ from \Cref{ass:S} by using its resolvent $T=J_{\gamma S}$. %
However, the update is implicit, so we will instead approximate $J_{\gamma S}$.
Given $z \in \R^d$ we seek $z' \in \R^d$ such that
\begin{equation*}
z' = J_{\gamma S}(z) = (\id + \gamma S)^{-1}z
   = (\id + \gamma A)^{-1}(z - \gamma Fz')
\end{equation*}
This can be approximated with a fixed point iteration of
\begin{equation}
Q_z: w \mapsto (\id + \gamma A)^{-1}(z - \gamma Fw)
\end{equation}
which is a contraction for small enough $\gamma$ since $F$ is Lipschitz continuous.
It follows from Banach's fixed-point theorem \cite{banach1922operations} that the sequence converges linearly.
We formalize this in the following theorem, which additionally applies when only stochastic feedback is available.
\begin{equation}\label{alg:OPS}
w^{t+1} = (\id + \gamma A)^{-1}(z - \gamma \hat F_\sigma(w^t, \xi_t))
\quad
\xi_{t} \sim \mathcal P
\end{equation}
\begin{thmbox}
\begin{lemmarep}\label{thm:banach}
Suppose \Cref{ass:A:maxmon,ass:F:Lips,,ass:stoch}.
Given $z \in \R^d$, the iterates generated by \eqref{alg:OPS} with $\gamma \in (0,\nicefrac{1}{L})$ converges to a neighborhood linearly, i.e.,
\begin{equation}
\mathbb E\big[\|w^{\tau}-J_{\gamma S}(z)\|^2\big] \leq (\gamma L)^{2\tau} \|w^{0}-w^\star\|^2 + \tfrac{\gamma^2}{(1- \gamma L)^2}\sigma^2.
\end{equation}
\end{lemmarep}
\end{thmbox}
\begin{appendixproof}
Let $\zeta^t = Fw^t - \hat F_\sigma(w^t, \xi_t)$. Then the stochastic update in \eqref{alg:OPS} can be written as 
\begin{equation}
w^{t+1} = (\id + \gamma A)^{-1}(z - \gamma Fw^t + \gamma \zeta^t)
\end{equation}
Let $w^\star \in \fix Q_z$ such that
\begin{equation}
\|w^{t+1}-w^\star\|^2 = \|w^{t+1}-Q_z(w^\star)\|^2.
\end{equation}
Due to (firmly) nonexpansiveness of $(\id + \gamma A)^{-1}$ when $A$ is maximally monotone we can go on as
\begin{align*}
\|w^{t+1}-Q_z(w^\star)\|^2 &= \|(\id + \gamma A)^{-1}(z - \gamma Fw^t + \gamma \zeta^t) - (\id + \gamma A)^{-1}(z - \gamma Fw^\star)\|^2  \\
&\leq \|(z - \gamma Fw^t + \gamma \zeta^t) - (z - \gamma Fw^\star)\|^2  \\
& = \gamma^2\|Fw^t-Fw^\star\|^2 + \gamma^2\|\zeta^t\|^2 + 2\gamma^2\braket{\zeta^t,Fw^\star - Fw^t} \\
& \leq \gamma^2 L^2\|w^t-w^\star\|^2 + \gamma^2\|\zeta^t\|^2 + 2\gamma^2\braket{\zeta^t,Fw^\star - Fw^t}
\numberthis\label{eq:contraction}
\end{align*}
where the last inequality follows from Lipschitz continuity of $F$.

Taking expectation and using unbiasedness and bounded variance from \Cref{ass:stoch} we get
\begin{equation}
\mathbb E\big[ \|w^{t+1}-w^\star\|^2 \mid \mathcal F_t \big] 
  \leq \gamma^2 L^2\|w^t-w^\star\|^2 + \gamma^2\sigma^2
\end{equation}
By law of total expectation
\begin{align*}
\mathbb E\big[ \|w^{\tau}-w^\star\| \big] 
  &\leq \gamma^2 L^2 \mathbb E\big[\|w^{\tau-1}-w^\star\|^2\big] + \gamma^2\sigma^2 \\
  &\leq \gamma^4 L^4 \mathbb E\big[\|w^{\tau-2}-w^\star\|^2\big] + \gamma^2(1+\gamma^2 L^2)\sigma^2 \\
  &\leq \cdots 
  \leq (\gamma L)^{2\tau} \mathbb E\big[\|w^{0}-w^\star\|^2\big] + \gamma^2\sigma^2\sum_{t=0}^{\tau-1}(\gamma L)^{2t} \\
  &\leq (\gamma L)^{2\tau} \|w^{0}-w^\star\|^2 + \tfrac{\gamma^2}{(1- \gamma L)^2}\sigma^2
\end{align*}
where the last inequality follows from $\sum_{t=0}^{\infty}a^{t} = \tfrac{1}{1-a}$ when $a < 1$.

By construction $\fix Q_z=\{J_{\gamma S}(z)\}$ 
which completes the proof.
\end{appendixproof}

The resulting update in \eqref{alg:OPS} is identical to GDA but crucially always steps from $z$.
We use this as a subroutine in \RAPP to get convergence under a cohypomonotone operator while only suffering a logarithmic factor in the rate.

\paragraph{Interpretation}\label{sec:ops:inter}

In the special case of the constrained minimax problem in \eqref{eq:minimax}, the application of the resolvent $J_{\gamma S}(z)$ is equivalent to solving the following optimization problem
\begin{equation}
\label{eq:ppm0}
\begin{split}
\min _{x' \in \mathcal{X}} \max _{y' \in \mathcal{Y}}&\Big\{\phi_\mu(x',y') := \phi(x', y')
+\frac{1}{2 \mu}\left\|x'-x\right\|^{2}-\frac{1}{2 \mu}\left\|y'-y\right\|^{2}\Big\}.
\end{split}
\end{equation}
for appropriately chosen $\mu \in (0,\infty)$. 
\eqref{alg:OPS} can thus be interpreted as solving a particular regularized subproblem.
Later we will drop this regularization to arrive at the Lookahead algorithm.

\begin{toappendix}
\section{Proofs for \Cref{sec:lastiterate} (Last iterate under cohypomonotonicity)}
\end{toappendix}
    \section{Last iterate under cohypomonotonicity}\label{sec:lastiterate}
    \nosectionappendix
    As stated in \Cref{sec:onestep}, we can obtain convergence using the approximate resolvent through \Cref{thm:KM:best}.
The convergence is provided in terms of the average, so additional work is needed for a last iterate result.
\ref{eq:IKM} iteration on the approximate resolvent (i.e. $\widetilde T_k(z) = J_{\gamma S}(z) + e^k$) becomes,
\begin{subequations}\label{eq:inexactResolvent}
\begin{align}
\bar z^k & = z^k - v^k 
\quad \text{with} \quad v^k \in \gamma S(\bar z^k) \label{eq:inexactResolvent:1} \\
z^{k+1} &= (1-\lambda)z^k + \lambda (\bar z^k + e^k) \label{eq:inexactResolvent:2}
\end{align}
\end{subequations}
with $\lambda \in (0,1)$ and $\gamma > 0$ and error $e^k \in \R^d$. Without error, \eqref{eq:inexactResolvent} reduces to relaxed proximal point
\begin{equation}\label{eq:RPP}
\tag{Relaxed PP}
z^{k+1} = (1-\lambda)z^k + \lambda J_{\gamma S}(z^k)
\end{equation}

For a last iterate result it remains to argue that the residual $\|J_{\gamma S}(z^k) - z^k\|$ is monotonically decreasing (up to an error we can control).
Showing monotonic decrease  is fairly straightforward if $\lambda=1$ (see \Cref{app:eq:lastiter} and the associated proof). However, we face additional complication due to the averaging, which is apparent both from the proof and the slightly more complicated error term in the following lemma.
\begin{thmbox}
\begin{lemmarep}\label{eq:lastiter}
If $S$ is $\rho$-comonotone with $\rho > -\frac{\gamma}{2}$ then \eqref{eq:inexactResolvent} satisfies for all $z^\star \in \zer S$, $$\|J_{\gamma S}(z^k) - z^k\|^2 \leq \|J_{\gamma S}(z^{k-1}) - z^{k-1}\|^2 + \delta_k(z^\star)$$
where $\delta_k(z):=4\|e^k\|(\|z^{k+1}-z\|+\|z^k-z\|)$.
\end{lemmarep}
\end{thmbox}
\begin{appendixproof}
Rearranging the update \eqref{eq:inexactResolvent:2} and subsequently using \eqref{eq:inexactResolvent:1},
\begin{equation}\label{eq:inexactResolvent:2:alt}
z^k-z^{k+1} = \lambda (z^k - \bar z^k - e^k) = \lambda (v^k - e^k).
\end{equation}
Since $\gamma S$ is $\tfrac{\rho}{\gamma}$-comonotone
\begin{equation}
\begin{split}
\tfrac{\rho}{\gamma} \|v^k - v^{k+1}\|^2 
  &\leq \braket{v^k - v^{k+1}, \bar z^k - \bar z^{k+1}} \\
\dueto{\eqref{eq:inexactResolvent:1}}&= \braket{v^k - v^{k+1}, z^k - v^k - (z^{k+1} - v^{k+1})} \\
&= \braket{v^k - v^{k+1}, z^k - z^{k+1}} - \|v^k - v^{k+1}\|^2 \\
\dueto{\eqref{eq:inexactResolvent:2:alt}}&= \lambda\braket{v^k - v^{k+1}, v^k - e^k} - \|v^k - v^{k+1}\|^2 \\
&= 
  \lambda\|v^k\|^2
  - \lambda\braket{v^{k+1}, v^k} 
  - \|v^k - v^{k+1}\|^2
  + \lambda\braket{v^{k+1} - v^k, e^k} 
\end{split}
\end{equation}
Rearranging
\begin{align*}
0
&\leq
  \lambda\|v^k\|^2
  - \lambda\braket{v^{k+1}, v^k} 
  - (1+\tfrac{\rho}{\gamma}) \|v^k - v^{k+1}\|^2
  + \lambda\braket{v^{k+1} - v^k, e^k} \\
&\leq
  \lambda\|v^k\|^2
  - \lambda\braket{v^{k+1}, v^k} 
  - \tfrac{\lambda}{2} \|v^k - v^{k+1}\|^2
  + \lambda\braket{v^{k+1} - v^k, e^k} \\
&=
  \lambda\|v^k\|^2
  - \tfrac{\lambda}{2}\|v^k\|^2
  - \tfrac{\lambda}{2}\|v^{k+1}\|^2
  + \lambda\braket{v^{k+1} - v^k, e^k} 
\numberthis \label{eq:last:descent}
\end{align*}
where the second inequality follows from observing that $(1+\tfrac{\rho}{\gamma}) > \tfrac{1}{2} > \tfrac{\lambda}{2}$ since $\lambda \in (0,1)$.
It remain to bound the error term. 
Since $\gamma S$ is $\nicefrac{1}{2}$-cohypomonotone the resolvent $J_{\gamma S}$ is nonexpansive. Thus,
\begin{equation}\label{eq:last:nonepansive}
\|\bar z^k - z^\star\| \leq \|z^k - z^\star\|.
\end{equation}
Using Cauchy-Schwarz and the triangle inequality,
\begin{align*}
\braket{v^{k+1} - v^k, e^k} 
&\leq \|e^k\|\|v^{k+1} - v^k\| \\
&\leq \|e^k\|(\|v^{k+1}\|+\|v^k\|) \\
&= \|e^k\|(\|\bar z^{k+1}-z^{k+1}\|+\|\bar z^k-z^k\|) \\
&\leq \|e^k\|(\|z^{k+1}-z^\star\|+\|z^k-z^\star\| + \|\bar z^{k+1}-z^\star\|+\|\bar z^k-z^\star|\|) \\
\dueto{\eqref{eq:last:nonepansive}}&\leq 2\|e^k\|(\|z^{k+1}-z^\star\|+\|z^k-z^\star\|) %
\numberthis \label{eq:last:error}
\end{align*}
Combining \eqref{eq:last:descent} and \eqref{eq:last:error},
\begin{equation}
\tfrac{1}{2}\|v^{k+1}\|^2 \leq \tfrac{1}{2}\|v^k\|^2 + 2\|e^k\|(\|z^{k+1}-z^\star\|+\|z^k-z^\star\|).
\end{equation}
Substituting in the resolvent using \eqref{eq:inexactResolvent:1} completes the proof.
\end{appendixproof}
\begin{toappendix}
\noindent The proof of \Cref{eq:lastiter} simplifies for $\lambda = 1$.
Consider one application of the inexact resolvent with error $e \in \R^d$,
\begin{equation}\label{app:eq:inexactResolvent}
z' = J_{\gamma S}(z) + e,
\end{equation}
where $\lambda \in (0,1)$ and $\gamma > 0$.
\begin{thmbox}
\begin{lemma}\label{app:eq:lastiter}
If $S$ is $\rho$-comonotone with $\rho > -\frac{\gamma}{2}$ then \eqref{app:eq:inexactResolvent} satisfies $\|J_{\gamma S}(z') - z'\| \leq \|J_{\gamma S}(z) - z\| + 2\|e\|$.
\end{lemma}
\end{thmbox}
\begin{proof}
Since $\gamma S$ is $\nicefrac{1}{2}$-cohypomonotone the resolvent $J_{\gamma S}$ is nonexpansive. Thus,
\begin{align*}
\|J_{\gamma S}(z') - z'\| 
  &= \|J_{\gamma S}(z') - J_{\gamma S}(z) -  e\| \\
\dueto{(triangle ineq.)}&\leq 
  \|J_{\gamma S}(z') - J_{\gamma S}(z)\| + \|e\| \\
\dueto{(nonexpansive)}&\leq 
  \|z' - z\| + \|e\| \\
\dueto{(triangle ineq.)}&\leq 
  \|J_{\gamma S}(z) - z\| + 2\|e\|
\end{align*}
This completes the proof.
\end{proof}

Furthermore, the iterates of \eqref{eq:inexactResolvent} are bounded in the following sense.
\begin{lemma}\label{lm:iters:bounded}
Consider the sequence $(z^k)_{k\in \mathbb N}$ generated by \eqref{eq:inexactResolvent} with $\lambda \in (0,1)$ and $\rho > -\frac{\gamma}{2}$.
Then for any $z^\star \in \zer S$,
\begin{equation}
\|z^{k+1} - z^\star\| \leq \|z^{0} - z^\star\| + \lambda\sum_{j=0}^k\|e^j\|.
\end{equation}
\end{lemma}
\begin{proof}
Since $\gamma S$ is $\nicefrac{1}{2}$-cohypomonotone the resolvent $J_{\gamma S}$ is nonexpansive. Thus,
\begin{equation}\label{eq:iters:bounded:nonepansive}
\|\bar z^k - z^\star\| \leq \|z^k - z^\star\|.
\end{equation}
We use the update rule
\begin{align*}
\|z^{k+1}-z^\star\| &= \|(1-\lambda)z^{k}+\lambda(\bar z^k + e^k)-z^\star\| \\
  &\leq \|(1-\lambda)(z^{k} - z^\star)+\lambda(\bar z^k - z^\star)\| + \lambda\|e^k\| \\
  &\leq (1-\lambda)\|z^{k} - z^\star\| + \lambda\|\bar z^k - z^\star\| + \lambda\|e^k\| \\
  \dueto{\eqref{eq:iters:bounded:nonepansive}}&\leq \|z^{k} - z^\star\| + \lambda\|e^k\|
  \numberthis \label{eq:iters:bounded}
\end{align*}
By recursively applying \eqref{eq:iters:bounded} we obtain the claim. 
\end{proof}
\end{toappendix}
The above lemma allows us to obtain last iterate convergence for \ref{eq:IKM} on the inexact resolvent by combing the lemma with \Cref{thm:KM:best}.
\begin{thmbox}
\begin{thmrep}[Last iterate of inexact resolvent]\label{thm:inexactResolvent:last}
Suppose \Cref{ass:S,ass:stoch} with $\sigma_k$.
Consider the sequence $(z^k)_{k\in \mathbb N}$ generated by \eqref{eq:inexactResolvent} with $\lambda \in (0,1)$ and $\rho > -\frac{\gamma}{2}$. 
Then, for all $z^\star \in \zer S$,
\begin{equation*}
\begin{split}
\EF[][[\|J_{\gamma S}(z^K) - z^{K}\|^2]]
&\leq 
  \frac{\|z^0 - z^\star \|^2 + \sum_{k=0}^{K-1}\varepsilon_k(z^\star)}{\lambda(1-\lambda) K} 
  \quad + \frac 1K \sum_{k=0}^{K-1} \sum_{j=k}^{K-1} \delta_j(z^\star),
\end{split}
\end{equation*}
where $\varepsilon_k(z) := 2\lambda \EF[][[\|e^k\|\|z^{k} - z \|]] + \lambda^2\EF[][[\|e^k\|^2]]$ and $\delta_k(z):=4\EF[][[\|e^k\|(\|z^{k+1}-z\|+\|z^k-z\|)]]$.
\end{thmrep}
\end{thmbox}
\begin{remark}
Notice that the rate in \Cref{thm:inexactResolvent:last} has \emph{no} dependency on $\rho$.
Specifically, it gets rid of the factor $\gamma/(\gamma + 2\rho)$ which \citet[Thm. 3.2]{gorbunov2022convergence} shows is unimprovable for PP.
\Cref{thm:inexactResolvent:last} requires that the iterates stays bounded. 
In \Cref{cor:inexactResolvent} we will assume bounded diameter for simplicity, but it is relatively straightforward to show that the iterates can be guaranteed to be bounded by controlling the inexactness (see \Cref{lm:iters:bounded}).
\end{remark}
\begin{appendixproof}
By taking $T=J_{\gamma S}$ in \Cref{thm:KM:best} we have
\begin{equation}\label{eq:inexactResolvent:average}
\frac{1}{K}\sum_{k=0}^{K-1}\EF[][[\|J_{\gamma S}(z^k) - z^k\|^2]] \leq \frac{\|z^0 - z^\star \|^2 + \sum_{k=0}^{K-1}\varepsilon_k(z^\star)}{\lambda(1-\lambda) K}.
\end{equation}
From \Cref{eq:lastiter} (and law of total expectation) we obtain,
\begin{equation}
K\EF[][[\|J_{\gamma S}(z^K) - z^{K}\|^2]] \leq \sum_{k=0}^{K-1}\EF[][[\|J_{\gamma S}(z^k) - z^{k}\|^2]] + \sum_{k=0}^{K-1} \sum_{j=k}^{K-1} \delta_j(z^\star).
\end{equation}
Dividing by $K$ and combining with \eqref{eq:inexactResolvent:average} yields the rate.
Noticing that $\fix J_{\gamma S} = \zer S$ completes the proof.
\end{appendixproof}

\begin{algorithm}[t]
\caption{Relaxed approximate proximal point method (RAPP)}\label{alg:inexactResolvent}
\begin{algorithmic}[1]
	\Require
		\(z^0  \in\R^d\)
		$\lambda \in (0,1)$,
		$\gamma \in(\lfloor-2 \rho\rfloor_{+}, \nicefrac 1L)$

\item[\algfont{Repeat} for \(k=0,1,\ldots\) until convergence]

\State 
  $w^0_k = z^k$
\ForAll{\(t=0,1,\ldots,\tau-1\)}
  \State $\xi_{k,t} \sim \mathcal P$
  \State
    $w^{t+1}_k = (\id + \gamma A)^{-1}(z^k - \gamma \hat F_{\sigma_k}(w^t_k, \xi_{k,t}))$
\EndFor
\State 
  $z^{k+1} = (1-\lambda)z^k + \lambda w^{\tau}_k$
\item[\algfont{Return}]
	\(z^{k+1}\)
\end{algorithmic}
\end{algorithm}
All that remains to get convergence of the explicit scheme in \RAPP, is to expand and simplify the errors $\varepsilon_k(z)$ and $\delta_k(z)$ using the approximation of the resolvent analyzed in \Cref{thm:banach}.
\begin{toappendix}
\begin{thmbox}
\begin{cor}[Explicit inexact stochastic resolvent]\label{cor:inexactResolvent:stoc}
Suppose \Cref{ass:S,ass:stoch} with $\sigma_k$ for all $k \in \mathbb N$.
Consider the sequence $(z^k)_{k\in \mathbb N}$ generated by \RAPP with $\rho > -\frac{\gamma}{2}$.
Then, for all $z^\star \in \zer S$ with $D := \sup_{j \in \mathbb N}\|z^j - z^\star\| < \infty$,
\begin{cornum}
\item \label{cor:inexactResolvent:stoc:best}
  with $\sigma_k^2 = \nicefrac{\sigma_0^2}{k^2}$ and $\tau = \frac{\log K}{\log (\nicefrac{1}{\gamma L})}$,
  \begin{equation*}
  \begin{split}
  \frac{1}{K}\sum_{i=0}^{K-1}\EF[][[\|J_{\gamma S}(z^k) - z^{k}\|^2]]
  &\leq 
    \frac{\|z^0 - z^\star \|^2}{\lambda(1-\lambda) K} 
    + \mathcal O\Big( 
        \max \big\{
                \tfrac{D^2}{(1-\lambda)K}, 
                \tfrac{\gamma D\sigma_0}{(1-\gamma L)(1-\lambda)K}
              \big\}
          \\ &\qquad 
              + \tfrac{\lambda D^2}{(1-\lambda)K} 
              + \tfrac{\lambda\gamma^2 \sigma_0^2}{(1-\gamma L)^2(1-\lambda)K^2}
     \Big). %
  \end{split}
  \end{equation*}
\item \label{cor:inexactResolvent:stoc:last}
  with $\sigma_k^2 = \nicefrac{\sigma_0^2}{k^3}$ and $\tau = \frac{\log K^2}{\log (\nicefrac{1}{\gamma L})}$,
  \begin{equation}
  \begin{split}
  \EF[][[\|J_{\gamma S}(z^K) - z^{K}\|^2]] 
    & \leq 
      \frac{\|z^0 - z^\star \|^2}{\lambda(1-\lambda) K}
      + \mathcal O\left( 
          \max \{\tfrac{D^2}{K}, \tfrac{8\gamma D \sigma_0}{(1-\gamma L)\sqrt{K}}\}
      \right) \\
      &\qquad + \mathcal O\Big( 
          \max \big\{
                  \tfrac{D^2}{(1-\lambda)K^2}, 
                  \tfrac{2\gamma D\sigma_0}{(1-\gamma L)(1-\lambda)K^{3/2}}
                \big\}
              \\ &\qquad
                + \tfrac{\lambda D^2}{(1-\lambda)K^{2}} 
                + \tfrac{\lambda\gamma^2 \sigma_0^2}{(1-\gamma L)^2(1-\lambda)K^3} 
      \Big) %
      \\
  \end{split}
  \end{equation}
\end{cornum}
\end{cor}
\end{thmbox}
\begin{remark}\label{rem:averagedAPPM:last}
The assumption on the noise $\sigma_k^2 = \sigma_0^2/n_k$ can be achieved by taking the batch size as $n_k$, i.e.
\begin{equation}
\hat F_{\sigma_k}(z,\xi) = \frac{1}{n_k} \sum_{i=0}^{n_k} \hat F_{\sigma_0}(z,\xi_i).
\end{equation}
This is clear by simple computation.
Observe that the random variable $X_i := \hat F_{\sigma}(z,\xi_i) - Fz$ is i.i.d. with $\operatorname {Var}(X_i)= \sigma^2$.
Then, the average, ${\overline {X}}_{n} = \tfrac{1}{n} (X_1 + \cdots + X_n)$, has a variance as follows
\begin{equation*}
\begin{split}
\operatorname {Var} ({\overline {X}}_{n})&=\operatorname {Var} ({\tfrac {1}{n}}(X_{1}+\cdots +X_{n})) 
  ={\frac {1}{n^{2}}}\operatorname {Var} (X_{1}+\cdots +X_{n})={\frac {n\sigma ^{2}}{n^{2}}}={\frac {\sigma ^{2}}{n}}.
\end{split}
\end{equation*}
We note that increasing batch size might be unfavorable in some applications, but the alternative of diminishing stepsize only leads to only asymptotic convergence of the last iterate (as in e.g. \citet{anonymous2023solving}).
\end{remark}
\begin{proof}
The theorem follows from combing \Cref{thm:banach} with \Cref{thm:KM:best,thm:inexactResolvent:last}.
Invoke \Cref{thm:KM:best,thm:inexactResolvent:last} with $e^k = w^\tau_k - J_{\gamma S}(z^k)$ and $\sigma = \sigma_k$ and note that the error $e^k$ can be bounded through \Cref{thm:banach} as
\begin{align*}
\EF[k][[\|e^k\|^2]] = \|w^\tau_k - J_{\gamma S}(z^k)\|^2 
&\leq \gamma^{2\tau} L^{2\tau} \|w^0_k - J_{\gamma S}(z^k)\|^2 + \tfrac{\gamma^2}{(1-\gamma L)^2}\sigma_k^2 \\
&= \gamma^{2\tau} L^{2\tau} \|z^k - J_{\gamma S}(z^k)\|^2 + \tfrac{\gamma^2}{(1-\gamma L)^2}\sigma_k^2.
\numberthis\label{eq:cor:inexactResolvent:last:noise}
\end{align*}
The former term can in turn be bounded through the triangle inequality
\begin{equation}
\|z^k - J_{\gamma S}(z^k)\| 
  \leq \|z^k - z^\star\| + \|J_{\gamma S}(z^k) - z^\star\|
  \leq 2\|z^k - z^\star\|
  \leq 2D.
\end{equation}
with $D := \sup_{j \in \mathbb N}\|z^j - z^\star\|$ and where the second last inequality follows from $z^\star \in \fix J_{\gamma S}$ and nonexpansiveness of $J_{\gamma S}$.
Plugging into \eqref{eq:cor:inexactResolvent:last:noise} we have,
\begin{equation}
\EF[k][[\|e^k\|^2]]
  \leq 4\gamma^{2\tau} L^{2\tau} D^2 + \tfrac{\gamma^2}{(1-\gamma L)^2}\sigma_k^2,
\end{equation}
and
\begin{equation}
\EF[k][[\|e^k\|]]
  \leq \sqrt{4\gamma^{2\tau} L^{2\tau} D^2 + \tfrac{\gamma^2}{(1-\gamma L)^2}\sigma_k^2}
  \leq \max \{2\gamma^{\tau} L^{\tau} D, \tfrac{\gamma}{1-\gamma L}\sigma_k\}.
\end{equation}
Substituting into the expression of $\delta_k(z^\star)$ and $\varepsilon_k(z^\star)$ yields,
\begin{align*}
\delta_k(z^\star) &\leq \max \{16\gamma^{\tau} L^{\tau} D^2, \tfrac{8\gamma}{1-\gamma L}\sigma_k D\} \\
\varepsilon_k(z^\star) &\leq \max \{4\lambda\gamma^{\tau} L^{\tau} D^2, \tfrac{2\lambda\gamma}{1-\gamma L}\sigma_k D\} + 4\lambda^2 \gamma^{2\tau} L^{2\tau} D^2 + \tfrac{\lambda^2\gamma^2}{(1-\gamma L)^2}\sigma_k^2.
\end{align*}
Consequently, with the choice $\sigma_k^2=\nicefrac{\sigma_0^2}{k^2}$,
\begin{equation}
\begin{split}
\frac{\sum_{k=0}^{K-1}\varepsilon_k(z^\star)}{\lambda(1-\lambda) K} 
  &\leq %
      \max \big\{
        \frac{4\gamma^{\tau} L^{\tau} D^2}{1-\lambda}, 
        \tfrac{2\gamma D\sigma_0}{(1-\gamma L)(1-\lambda)K}
      \big\} 
      + \frac{4\lambda \gamma^{2\tau} L^{2\tau} D^2}{{1-\lambda}} 
      + \tfrac{\lambda\gamma^2 \sigma_0^2}{(1-\gamma L)^2(1-\lambda)K^2}
    .
\end{split}
\end{equation}
We ideally want the terms involving $\tau$ to be of order $\mathcal O(1/K)$.
\begin{equation}
1/a^\tau = 1/K \Longleftrightarrow a^\tau = K \Longleftrightarrow \tau \log a = \log K \Longleftrightarrow \tau = \frac{\log K}{\log a}
\end{equation}
Choosing $a=1/\gamma L$ it thus suffice to pick $\tau = \frac{\log K}{\log (\nicefrac{1}{\gamma L})}$ in order to have $\gamma^\tau L^\tau = 1/K$.
Plugging into the average iterate result of \Cref{thm:KM:best} yields the claim in \Cref{cor:inexactResolvent:stoc:best}.

Additionally, with the choice $\sigma_k^2=\nicefrac{\sigma_0^2}{k^3}$,
\begin{equation}\label{eq:lastiter:errors}
\begin{split}
\frac{\sum_{k=0}^{K-1}\varepsilon_k(z^\star)}{\lambda(1-\lambda) K} 
  &\leq %
      \max \big\{
        \frac{4\gamma^{\tau} L^{\tau} D^2}{1-\lambda}, 
        \tfrac{2\gamma D\sigma_0}{(1-\gamma L)(1-\lambda)K^{3/2}}
      \big\} 
      + \frac{4\lambda \gamma^{2\tau} L^{2\tau} D^2}{{1-\lambda}} 
      + \tfrac{\lambda\gamma^2 \sigma_0^2}{(1-\gamma L)^2(1-\lambda)K^3}
    \\
\frac 1K \sum_{k=0}^{K-1} \sum_{j=k}^{K-1} \delta_j(z^\star) 
  &\leq %
      \max \{K16\gamma^{\tau} L^{\tau} D^2, \tfrac{8\gamma D \sigma_0}{(1-\gamma L)\sqrt{K}}\}
    .
\end{split}
\end{equation}
In order for the terms involving $\tau$ to be of order $\mathcal O(1/K)$ we need $\tau$ to be slightly larger.
\begin{equation}
K/a^\tau = 1/K \Longleftrightarrow a^\tau = K^2 \Longleftrightarrow \tau \log a = \log K^2 \Longleftrightarrow \tau = \frac{\log K^2}{\log a}
\end{equation}
Choosing $a=1/\gamma L$ it thus suffice to pick $\tau = \frac{\log K^2}{\log (\nicefrac{1}{\gamma L})}$ in order to have $K\gamma^\tau L^\tau = 1/K$.
Plugging \eqref{eq:lastiter:errors} into the last iterate result of \Cref{thm:inexactResolvent:last} completes the proof.
\end{proof}
\end{toappendix}
\begin{correp}[Explicit inexact resolvent]\label{cor:inexactResolvent}
Suppose \Cref{ass:S} holds.
Consider the sequence $(z^k)_{k\in \mathbb N}$ generated by \RAPP with deterministic feedback %
 and $\rho > -\frac{\gamma}{2}$.
Then, for all $z^\star \in \zer S$ with $D := \sup_{j \in \mathbb N}\|z^j - z^\star\| < \infty$,
\begin{cornum}
\item \label{cor:inexactResolvent:best}
  with $\tau = \frac{\log K}{\log (\nicefrac{1}{\gamma L})}$:
  $\frac{1}{K}\sum_{i=0}^{K-1}\|J_{\gamma S}(z^k) - z^{k}\|^2
  = 
    \mathcal O\Big( 
    \frac{\|z^0 - z^\star \|^2}{\lambda(1-\lambda) K} 
              +  \tfrac{D^2}{(1-\lambda)K}
     \Big). 
  $
\item \label{cor:inexactResolvent:last}
  with $\tau = \frac{\log K^2}{\log (\nicefrac{1}{\gamma L})}$: 
  $\|J_{\gamma S}(z^K) - z^{K}\|^2
    = 
      \mathcal O\Big( 
      \frac{\|z^0 - z^\star \|^2}{\lambda(1-\lambda) K}
          + \tfrac{D^2}{K}
          + \tfrac{D^2}{(1-\lambda)K^2}, 
      \Big).
      $
\end{cornum}
\end{correp}
\begin{remark}
\Cref{cor:inexactResolvent:last} implies an oracle complexity of ${\mathcal O}\big(\log(\varepsilon^{-2})\varepsilon^{-1}\big)$ for ensuring that the last iterate satisfies $\|J_{\gamma S}(z^K) - z^K\|^2\leq \varepsilon$.
A stochastic extension is provided in \Cref{cor:inexactResolvent:stoc} by taking the batch size increasing.
Notice that \RAPP, for $\tau=2$ inner steps, reduces to \ref{eq:EG+} in the unconstrained case where $A\equiv 0$.
\end{remark}
\begin{appendixproof}
The claim follows directly from \Cref{cor:inexactResolvent:stoc} as a special case with $\sigma_0 = 0$.
\end{appendixproof}

\begin{toappendix}
\section{Proofs for \Cref{sec:lookahead} (Analysis of Lookahead)}
\end{toappendix}
    \section{Analysis of Lookahead}\label{sec:lookahead}
    \nosectionappendix
    
\begin{figure*}[t]
\centering
\includegraphics[width=0.5\textwidth]{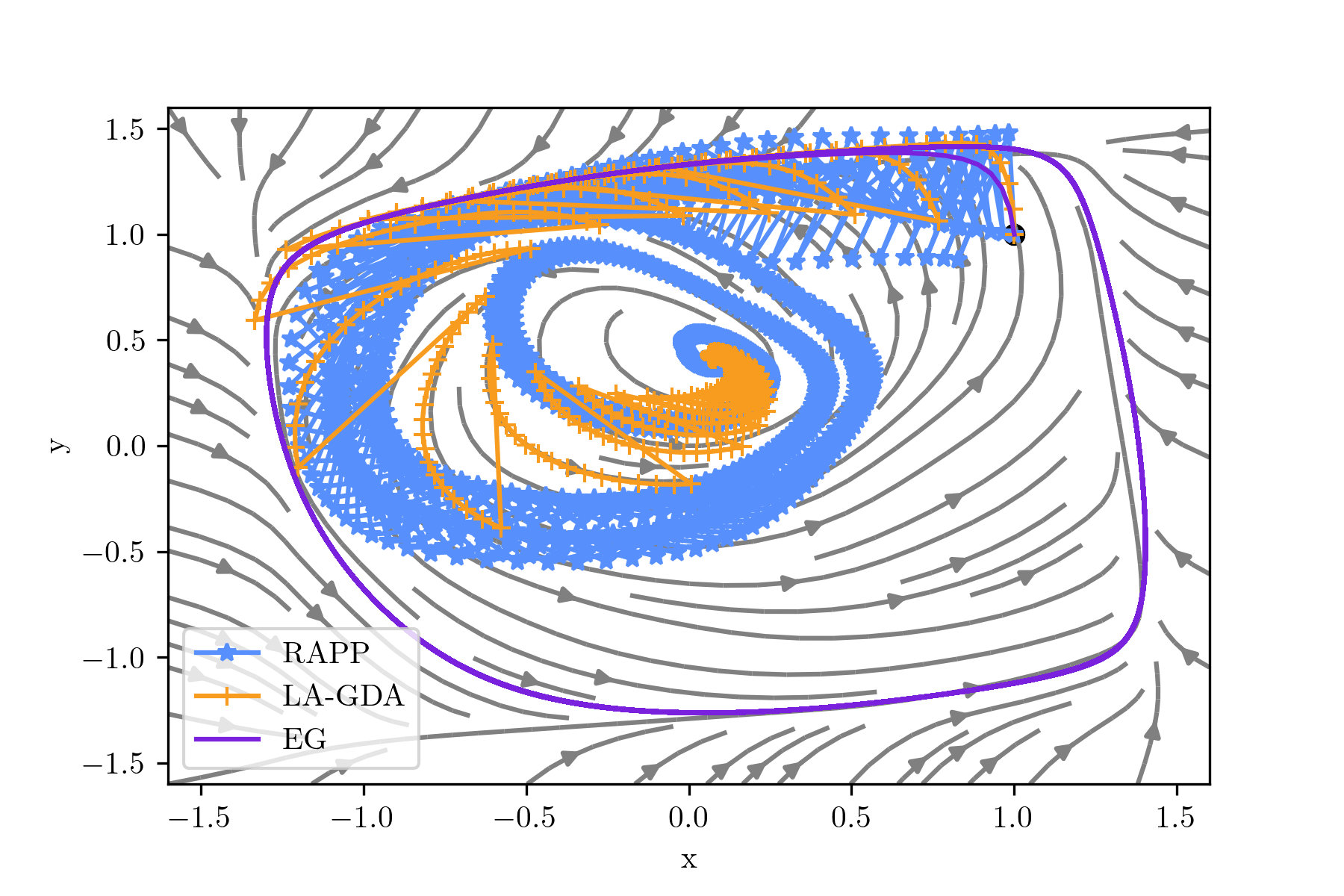}%
\includegraphics[width=0.5\textwidth]{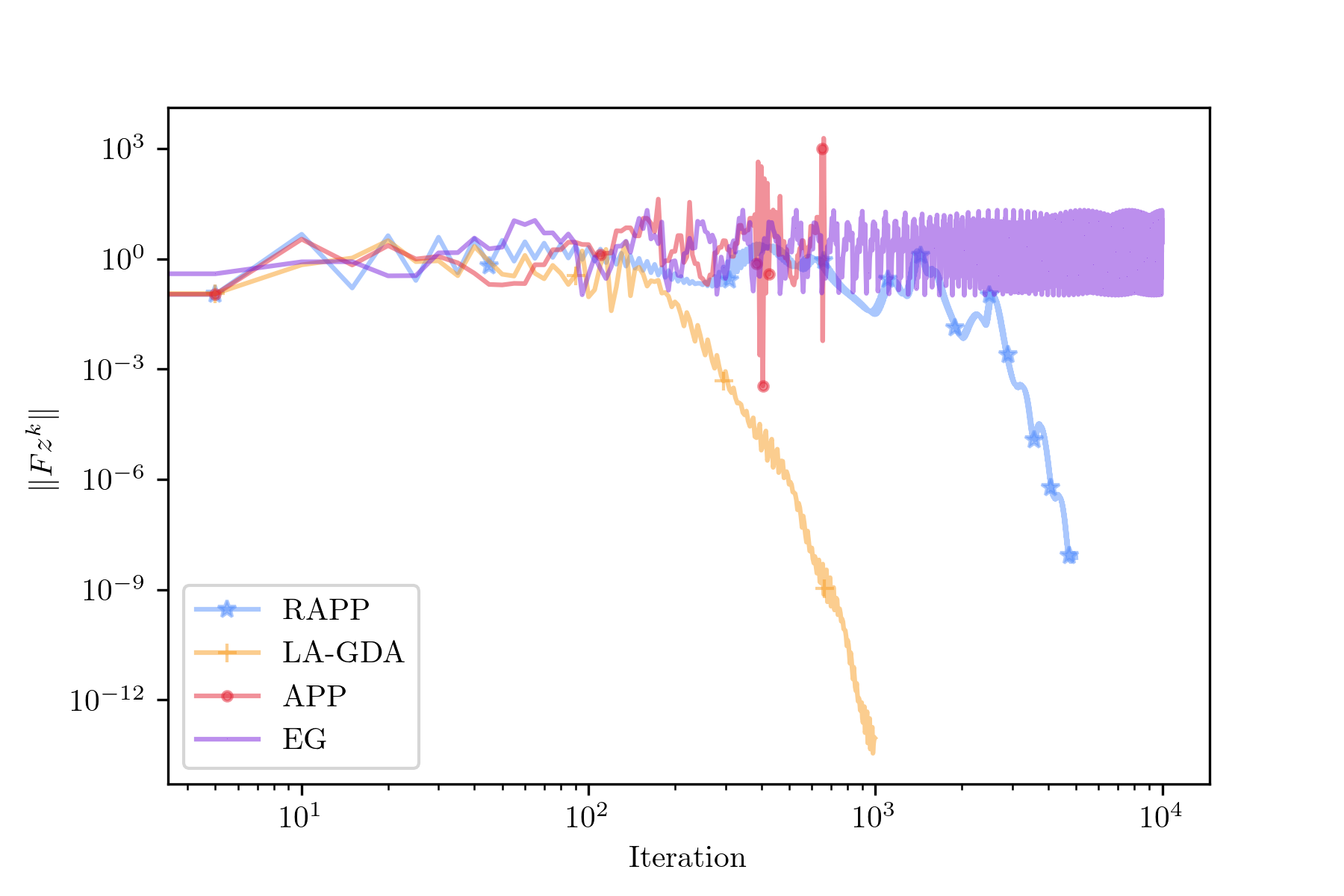}%
\caption{\ref{eq:lookahead} and \RAPP can converge for \citet[Ex. 5.2]{hsieh2021limits}. Interestingly, we can set the stepsize $\gamma$ larger than $\nicefrac 1L$ while \RAPP remains stable. Approximate proximal point (APP) with the same stepsize diverges (the iterates of APP are deferred to \Cref{fig:APPM:iterates}).
In this example, it is apparent from the rates, that there is a benefit in replacing the conservative inner update in \RAPP with GDA in \ref{eq:lookahead} as explored in \Cref{sec:lookahead}.}
\label{fig:forsaken}
\end{figure*}

The update in \RAPP leads to a fairly conservative update in the inner loop, since it corresponds to optimizing a highly regularized subproblem as noted in \Cref{sec:ops:inter}.
Could we instead replace the optimization procedure with gradient descent ascent (GDA)?
If we replace the inner optimization routine we recover what is known as the Lookahead (LA) algorithm
\begin{equation}
\label{eq:lookahead}
\tag{LA-GDA}
\begin{split}
w_k^0 &= z^{k} \\
w_k^{t+1} &= w_k^{t} - \gamma F w_k^{t} \quad \forall t = 0,...,\tau-1 \\
z^{k+1} &= (1-\lambda) z^{k} + \lambda w_k^{\tau}
\end{split}
\end{equation}
We empirically demonstrate that this scheme can converge for nonmonotone problems for certain choices of parameters (see \Cref{fig:forsaken}).
However, what global guarantees can we provide theoretically?

It turns out that for \ref{eq:lookahead} with two inner steps ($\tau=2$) we have an affirmative answer.
After some algebraic manipulation it is not difficult to see that the update can be simplified as follows
\begin{equation}\label{eq:LA:k2:subs}
\begin{split}
z^{k+1} & = \tfrac 12(z^{k} - 2\lambda\gamma F z^k) + \tfrac 12(z^k - 2\lambda\gamma F(z^{k} - \gamma F z^k)).
\end{split}
\end{equation}
This is the average of GDA and EG+ (when $\lambda \in (0,\nicefrac 12)$).
This observation allows us to show convergence under cohypomonotonicity.
This positive result for nonmonotone problems partially explains the stabilizing effect of \ref{eq:lookahead}.
\begin{thmbox}
\begin{thmrep}\label{thm:LA:k2}
Suppose \Cref{ass:S} holds. 
Consider the sequence $(z^k)_{k\in \mathbb N}$ generated by \ref{eq:lookahead} with $\tau=2$, $\gamma \leq \nicefrac 1L$ and $\lambda \in (0,\nicefrac{1}{2})$.
Furthermore, suppose that
\begin{equation}\label{thm:eq:LA:k2:conditions}
2\rho > -(1-2\lambda)\gamma
\quad \text{and} \quad
2\rho \geq 2\lambda\gamma-(1-\gamma^2L^2)\gamma.
\end{equation}
Then, for all $z^\star \in \zer F$,
\begin{equation}\label{eq:LA:k2:rate}
\frac{1}{K} \sum_{k=0}^{K-1}\|F\bar z^k\|^2  \leq 
\frac{\|z^0 - z^\star\|^2}{\lambda\gamma\big((1-2\lambda)\gamma + 2\rho\big)K}.
\end{equation}
\end{thmrep}
\end{thmbox}
\begin{remark}
For $\lambda \rightarrow 0$ and $\gamma=\nicefrac cL$ where $c \in (0,\infty)$, sufficient condition reduces to $\rho \geq -\gamma(1-\gamma^2L^2)/2 = -\nicefrac{c(1-c^2)}{2L}$, of which the minimum is attained with $c=\nicefrac{1}{\sqrt{3}}$, leading to the requirement $\rho \geq -\nicefrac{1}{3 \sqrt{3}L}$.
A similar statement is possible for $z^k$.
Thus, \eqref{eq:lookahead} improves on the range of $\rho$ compared with EG (see \Cref{tbl:overview}).
\end{remark}
\begin{appendixproof}
For $\tau=2$ we can write \eqref{eq:lookahead} as
\begin{equation}
\begin{split}
z^{k+1/3} &= z^{k} - \gamma F z^k \\
z^{k+2/3} &= z^{k+1/3} - \gamma F z^{k+1/3}\\
z^{k+1} &= (1-\lambda) z^{k} + \lambda z^{k+2/3}
\end{split}
\end{equation}
The proof relies on the simplified form of the update rule \eqref{eq:LA:k2:subs}, which can be obtain as follows
\begin{align*}
z^{k+1} &= (1-\lambda) z^{k} + \lambda z^{k+2/3} \\
    & = (1-\lambda) z^{k} + \lambda (z^{k+1/3} - \gamma F z^{k+1/3}) \\
    & = (1-\lambda) z^{k} + \lambda (z^{k} - \gamma F z^k - \gamma F z^{k+1/3}) \\
    & = z^{k} - \lambda\gamma F z^k - \lambda\gamma F(z^{k} - \gamma F z^k) \\
    & = \tfrac 12(z^{k} - 2\lambda\gamma F z^k) + \tfrac 12(z^k - 2\lambda\gamma F(z^{k} - \gamma F z^k)).
    \numberthis \label{eq:LA:k2:update}
\end{align*}
Define the following operators with $\beta = 2\lambda$
\begin{subequations}
\begin{align}
\operatorname{EG^+}(z) &= z - \beta\gamma F (z - \gamma Fz) \label{eq:EG+:operator}\\
\operatorname{GDA}(z) &= z - \beta\gamma Fz \label{eq:GDA:operator}
\end{align}
\end{subequations}
Then, using \eqref{eq:LA:k2:update}, \ref{eq:lookahead} with $\tau=2$ can be written as
\begin{equation}
z^{k+1} = \tfrac{1}{2}\operatorname{GDA}(z^k) + \tfrac{1}{2}\operatorname{EG^+}(z^k)
\end{equation}
One step of the update can be bounded as
\begin{equation}\label{eq:LA:k2:one}
\|z^{k+1} - z^\star \|^2 
= \|\tfrac{1}{2}\operatorname{GDA}(z^k) + \tfrac{1}{2}\operatorname{EG^+}(z^k) - z^\star \|^2 
\leq \tfrac{1}{2}\|\operatorname{GDA}(z^k) - z^\star \|^2 + \tfrac{1}{2}\|\operatorname{EG^+}(z^k) - z^\star \|^2,
\end{equation}
where we have used Young's inequality.
The first term can be expanded
\begin{equation}\label{eq:LA:k2:GDA}
\|\operatorname{GDA}(z^k) - z^\star \|^2 
= \|z^k - z^\star \|^2 + \beta^2\gamma^2\|Fz^k \|^2 - 2\beta\gamma \langle Fz^k, z^k - z^\star\rangle
\end{equation}
For the second term of \eqref{eq:LA:k2:one} we will need to bound the following inner product
\begin{align*}
\langle \gamma F\bar z^k, z^k - \bar z^k\rangle 
  &= \tfrac{\gamma^2}{2} \|F\bar z^k\|^2 - \tfrac{1}{2} \|\gamma F\bar z^k - (z^k - \bar z^k)\|^2 + \tfrac{1}{2} \|\bar z^k - z^k\|^2 \\
   \dueto{\eqref{eq:EG+:operator}}&= \tfrac{\gamma^2}{2} \|F\bar z^k\|^2 - \tfrac{\gamma^2}{2} \|F\bar z^k - Fz^k\|^2 + \tfrac{1}{2} \|\bar z^k - z^k\|^2 \\
   \dueto{(\Cref{ass:F:Lips})}&\geq \tfrac{\gamma^2}{2} \|F\bar z^k\|^2 + \tfrac{1}{2}(1 - \gamma^2L^2) \|\bar z^k - z^k\|^2
   \numberthis\label{eq:LA:k2:EG+:inner}.
\end{align*}
Consequently,
\begin{align*}
  \gamma\langle F\bar z^k, z^k - z^\star \rangle
      {}={}&
  \gamma\langle F\bar z^k, \bar z^k - z^\star \rangle
      {}+{}
  \gamma\langle F\bar z^k, z^k - \bar z^k \rangle
  \\
      \dueto{\eqref{eq:LA:k2:EG+:inner}}\leq{}&
  \gamma\langle F\bar z^k, \bar z^k - z^\star \rangle
  -\tfrac{\gamma^2}{2}\|F\bar z^k\|^2 
  - \tfrac{1}{2}(1-\gamma^2L^2) \|\bar z^k-z^k\|^2.
  \numberthis\label{eq:LA:k2:EG+:inner2}
\end{align*}
Finally,
\begin{align*}
    \|\operatorname{EG^+}(z^k)- z^\star\|^2 
        {}={}&
    \|z^{k}- z^\star\|^2 
        {}+{}
    \beta^2\gamma^2\|F\bar z^k\|^2
        {}-{}
    2\beta\gamma\langle F\bar z^k, z^k - z^\star \rangle
    \\
    \dueto{\eqref{eq:LA:k2:EG+:inner2}}\leq{}&
    \|z^{k}- z^\star\|^2 
    - \beta(1-\beta)\gamma^2\|F\bar z^k\|^2
    - \beta(1-\gamma^2L^2) \|\bar z^k-z^k\|^2
    -2\beta\gamma\langle F\bar z^k, \bar z^k - z^\star \rangle \\
\dueto{\eqref{eq:EG+:operator}}{}={}&
    \|z^{k}- z^\star\|^2 
    - \beta(1-\beta)\gamma^2\|F\bar z^k\|^2
    - \beta(1-\gamma^2L^2)\gamma^2 \|Fz^k\|^2
    -2\beta\gamma\langle F\bar z^k, \bar z^k - z^\star \rangle
    \numberthis\label{eq:LA:k2:EG+}
\end{align*}
Using \eqref{eq:LA:k2:GDA} and \eqref{eq:LA:k2:EG+} in \eqref{eq:LA:k2:one}, we have
\begin{align*}
2\|z^{k+1} - z^\star \|^2  
  &\leq 2\|z^{k} - z^\star \|^2 
    + \beta^2\gamma^2\|Fz^k \|^2 - 2\beta\gamma \langle Fz^k, z^k - z^\star\rangle \\
    &\quad - \beta(1-\beta)\gamma^2\|F\bar z^k\|^2
    - \beta(1-\gamma^2L^2)\gamma^2 \|Fz^k\|^2
    -2\beta\gamma\langle F\bar z^k, \bar z^k - z^\star \rangle \\
  \dueto{(\Cref{ass:F:cohypo})}&\leq 2\|z^{k} - z^\star \|^2 
    - \beta\gamma\big((1-\gamma^2L^2)\gamma + 2\rho - \beta\gamma\big)\|Fz^k \|^2 \\
    &\quad - \beta\gamma\big((1-\beta)\gamma + 2\rho\big)\|F\bar z^k\|^2 
  \numberthis\label{eq:LA:k2:descent}
\end{align*}
To get a recursion it thus suffice to require
\begin{equation}\label{eq:LA:k2:conditions}
(1-\beta)\gamma + 2\rho > 0 
\quad \text{and} \quad
(1-\gamma^2L^2)\gamma + 2\rho - \beta\gamma \geq 0.
\end{equation}
Rearranging and telescoping \eqref{eq:LA:k2:descent} achieves the claimed rate.
Rearranging \eqref{eq:LA:k2:conditions} completes the proof.
\end{appendixproof}
For larger $\tau$, \ref{eq:lookahead} does not necessarily converge (see \Cref{fig:LA:F} for a counterexample).
We next ask what we would require of the base optimizer to guarantee convergence for any $\tau$.
To this end, we replace the inner iteration with some abstract algorithm $\operatorname{Alg}: \R^d \rightarrow \R^d$, i.e.
\begin{equation}
\label{eq:LA}
\tag{LA}
\begin{split}
w_k^0 &= z^{k} \\
w_k^{t+1} &= \operatorname{Alg}(w_k^{t}) \quad \forall t = 0,...,\tau-1 \\
z^{k+1} &= (1-\lambda) z^{k} + \lambda w_k^{\tau}
\end{split}
\end{equation}
Convergence follows from quasi-nonexpansiveness.
\begin{thmbox}
\begin{thmrep}\label{thm:LA:nonexpansive}
Suppose $\operatorname{Alg}: \R^d \rightarrow \R^d$ is quasi-nonexpansive.
Then $(z^{k})_{k \in \mathbb N}$ generated by \eqref{eq:LA} converges to some $z^\star \in \fix \operatorname{Alg}$.
\end{thmrep}
\end{thmbox}
\begin{remark}\label{rem:LA:outer}
Even though the base optimizer $\operatorname{Alg}$ might not converge (since nonexpansiveness is not sufficient), \Cref{thm:LA:nonexpansive} shows that the outer loop converges.
Interestingly, this aligns with the benefit observed in practice of using the outer iteration of Lookahead (see \Cref{fig:adam}).
\end{remark}
\begin{appendixproof}
By the composition rule \citep[Prop. 4.49(ii)]{Bauschke2017Convex} $\operatorname{Alg}^t$ is also nonexpansive.
Since $(z^{k})_{k \in \mathbb N}$ can be seen as a Krasnosel'skiĭ-Mann iteration of a quasi-nonexpansive operator the iterates converges to $z^\star \in \fix \operatorname{Alg}^t$ by \Cref{thm:KM:best} with $\varepsilon_k=0$, i.e. $\|z^k - z^\star\| \overset{k \rightarrow \infty}{\longrightarrow} 0$.
By \citet[Prop. 4.49(i)]{Bauschke2017Convex} it also follows that $\fix \operatorname{Alg}^t = \fix \operatorname{Alg}$, which completes the proof.
\end{appendixproof}

\begin{figure*}[t]
\centering
\includegraphics[width=0.5\textwidth]{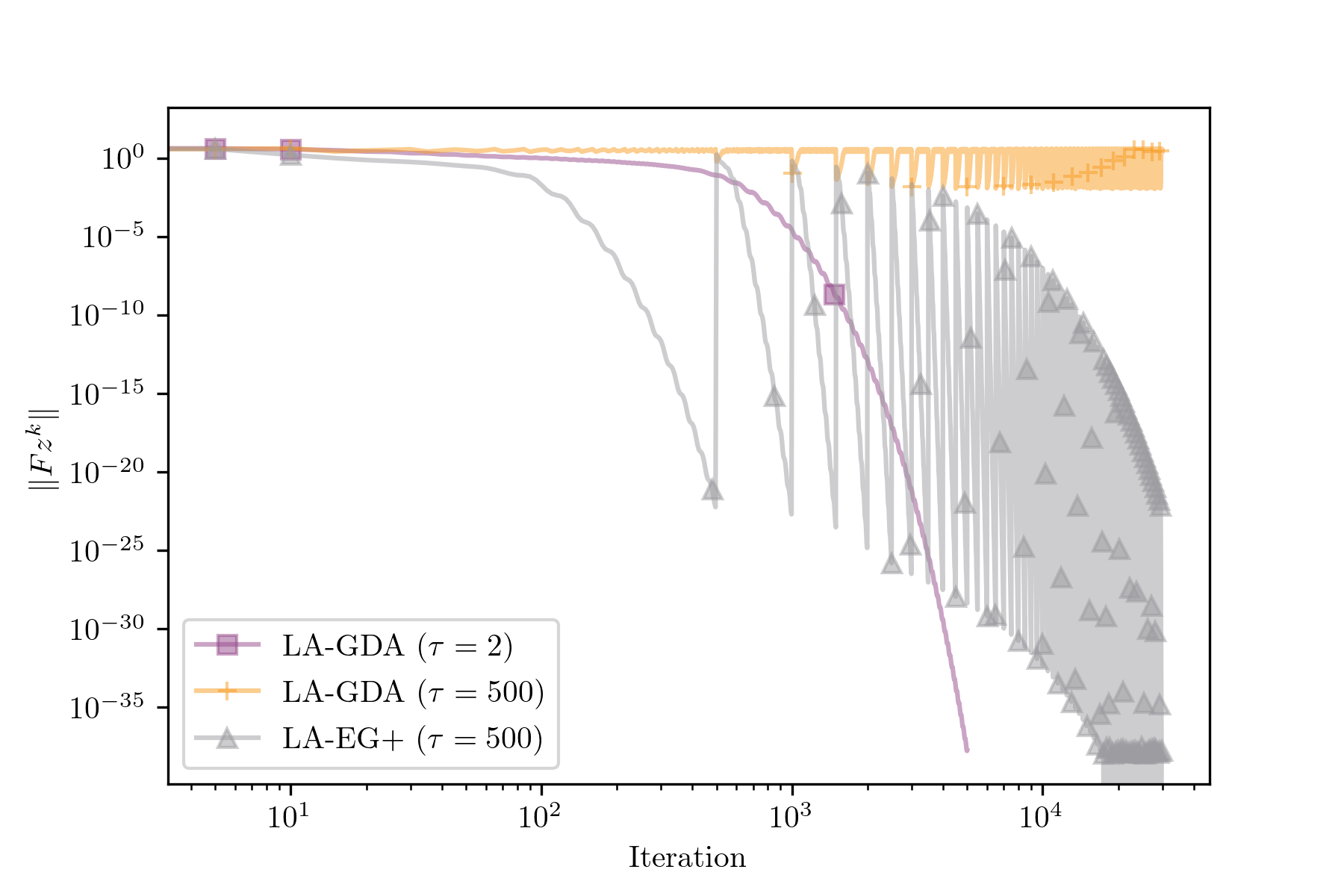}%
\includegraphics[width=0.5\textwidth]{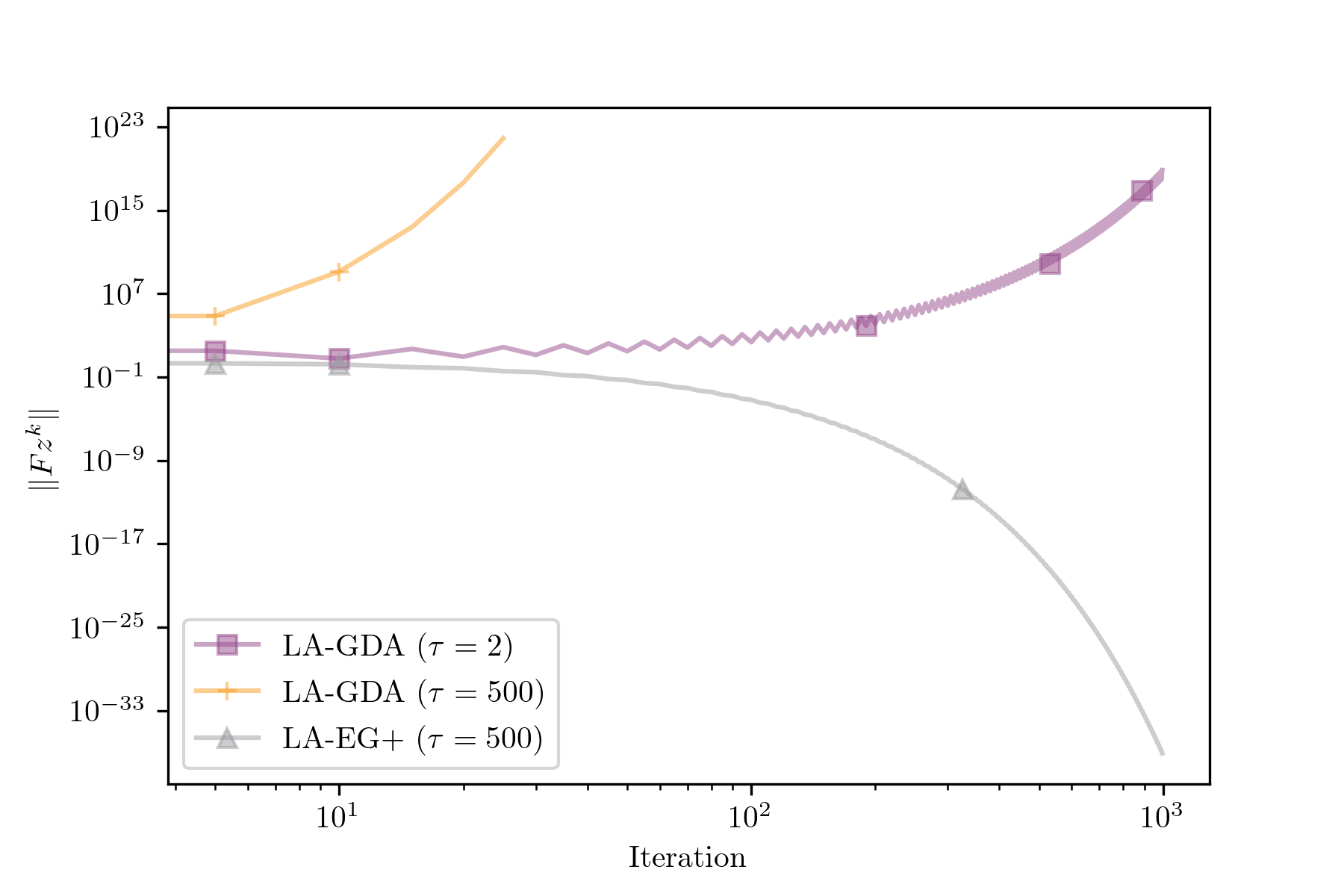}%
\caption{We test the Lookahead variants on 
\citet[Ex. 3(iii)]{pethick2022escaping} where $\rho \in (\nicefrac{-1}{8L},\nicefrac{-1}{10L})$  (left)
and \citet[Ex. 5]{pethick2022escaping} with $\rho = \nicefrac{-1}{3}$ (right).
For the left example \ref{eq:lookahead} (provably) converges for $\tau=2$, but may be nonconvergent for larger $\tau$ as illustrate.
Both variants of \ref{eq:lookahead} diverges in the more difficult example on the right, while \ref{eq:LA-CEG+} in contrast provably converges.
It seems that \ref{eq:LA-CEG+} trades off a constant slowdown in the rate for convergence in a larger class.
}
\label{fig:LA:F}
\end{figure*}

\paragraph{Cocoercive}
From \Cref{thm:LA:nonexpansive} we almost immediately get converge of \ref{eq:lookahead} for coercive problems since $V=\id - \gamma F$ is nonexpansive iff $\gamma F$ is $\nicefrac 12$-cocoercive.
\begin{thmbox}
\begin{correp}\label{thm:LA:cocoercive}
Suppose $F$ is $\nicefrac 1L$-cocoercive.
Then $(z^{k})_{k \in \mathbb N}$ generated by \ref{eq:lookahead} with $\gamma \leq \nicefrac 2L$ converges to some $z^\star \in \zer F$.
\end{correp}
\end{thmbox}
\begin{appendixproof}
If $F$ is $\nicefrac{1}{L}$-cocoercive then $\gamma F$ is $\nicefrac 12$-cocoercive given $\gamma \leq \nicefrac{2}{L}$, which in turn implies that $V=\id - \gamma F$ is nonexpansive.
The claim follows from \Cref{thm:LA:nonexpansive} and by observing that $\fix V = \zer F$.
\end{appendixproof}

\begin{remark}
\Cref{thm:LA:cocoercive} can trivially be extended to the constrained case by observing that also $V=(\id + \gamma A)^{-1}(\id - \gamma F)$ is nonexpansive when $A$ is maximally monotone.
As a special case this captures constrained convex and gradient Lipschitz minimization problems.
\end{remark}

\paragraph{Monotone}
When only monotonicity and Lipschitz holds we may instead consider the following extragradient based version of Lookahead (first empirically investigated in \citet{chavdarova2020taming})
\begin{equation}
\label{eq:LA-EG}
\tag{LA-EG}
\begin{split}
w_k^0 &= z^{k} \\
w_k^{t+1} &= \operatorname{EG}(w^k_t) \quad \forall t = 0,...,\tau-1 \\
z^{k+1} &= (1-\lambda) z^{k} + \lambda w_k^{\tau}
\end{split}
\end{equation}
where $\operatorname{EG}(z) = z - \gamma F(z - \gamma F z)$.
We show in \Cref{thm:FBF} that the $\operatorname{EG}$-operator of the inner loop is quasi-nonexpansive, which implies convergence of \ref{eq:LA-EG} through \Cref{thm:LA:nonexpansive}.
\Cref{thm:FBF} extends even to cases where $A \not \equiv 0$ by utilizing the forward-backward-forward construction of \citet{tseng1991applications}.
This providing the first global convergence guarantee for Lookahead beyond bilinear games.

\begin{toappendix}
Consider the forward-backward-forward (FBF) method of \cite{tseng1991applications}. 
We can write one step as follows
\begin{subequations}\label{eq:FBF}
\begin{align}
\z &= (\id + \gamma A)^{-1}\HC[][z] \label{eq:FBF:zbar} \\
\operatorname{FBF}(z) &= z - \left(\HC[][z] - \HC[][\z]\right) \label{eq:FBF:z}
\end{align}
\end{subequations}
where $\HC[][] = \id - \gamma F$. The extragradient method is obtained as a special case when $A\equiv 0$.

\begin{thm}\label{thm:FBF}
If $A\colon \R^d\rightrightarrows\R^d$ is maximally monotone and $F\colon \R^d\rightarrow\R^d$ is monotone and $L$-Lipschitz continuous then the operator \eqref{eq:FBF} with $\gamma \leq \nicefrac 1L$ is quasi-nonexpansive.
Furthermore, $\fix \operatorname{FBF} = \zer S$ with $S:=A+F$.
\end{thm}
\begin{proof}
By \(\nicefrac12\)-cocoercivity from \Cref{lm:Main:H:properties} we obtain
  \begin{align*}
    \langle \HC[][\z] - \HC[][z], z - z^\star \rangle
        {}={}&
    \langle \HC[][\z] - \HC[][z], \z - z^\star \rangle
        {}+{}
    \langle \HC[][\z] - \HC[][z], z - \z \rangle
    \\
    \dueto{(\Cref{lm:Main:H:properties})}\leq{}&
      \langle \HC[][\z] - \HC[][z], \z - z^\star \rangle
      -\tfrac12\|\HC[][\z]-\HC[][z]\|^2 
      - \tfrac12 (1-\gamma^2L^2) \|\bar z-z\|^2 
    \\
    \dueto{(monotone)}\leq{}&
      -\tfrac12\|\HC[][\z]-\HC[][z]\|^2 
      - \tfrac12 (1-\gamma^2L^2) \|\bar z-z\|^2
      \numberthis\label{eq:FBF:inner2}
\end{align*}
The operator in \eqref{eq:FBF:z} satisfies
\begin{align*}
    \|\operatorname{FBF}(z) - z^\star\|^2 
        {}={}&
    \|z- z^\star\|^2 
        {}+{}
    \|\HC[][\z] - \HC[][z]\|^2
        {}+{}
    2\langle \HC[][\z] - \HC[][z], z - z^\star \rangle
    \\
    \dueto{\eqref{eq:FBF:inner2}}
        {}\leq{}&
    \|z- z^\star\|^2 %
    - (1-\gamma^2L^2) \|\bar z-z\|^2 
\end{align*}
where the last term is negative due to $\gamma \leq \nicefrac 1L$.
Recognizing the definition of quasi-nonexpansive completes the proof.
\end{proof}

\end{toappendix}

\paragraph{Cohypomonotone}
For cohypomonotone problems large $\tau$ may prevent \ref{eq:lookahead} from converging (see \Cref{fig:LA:F} for a counterexample).
Therefore we propose replacing the inner optimization loop in \ref{eq:lookahead} with the method proposed in \citep[Alg. 1]{pethick2022escaping}.
Let $H=\id - \gamma F$.
We can write one step of the inner update with $\alpha\in (0,1)$ as
\begin{equation}
\label{eq:CEG+}
\begin{split}
\operatorname{CEG^+}(w) =w+2\alpha(H \bar{w}-H w)
\quad \text{with} \quad 
\bar{w} &=\left(\mathrm{id}+\gamma A\right)^{-1}Hw.
\end{split}   
\end{equation}
The usefulness of the operator $\operatorname{CEG^+}: \R^d \rightarrow \R^d$ comes from the fact that it is quasi-nonexpansive under \Cref{ass:S} (see \Cref{thm:EG+}).
Thus, \Cref{thm:LA:nonexpansive} applies even when $F$ is only cohypomonotone if we make the following modification to \ref{eq:lookahead}
\begin{equation}
\tag{LA-CEG+}
\label{eq:LA-CEG+}
\begin{split}
w_k^0 &= z^{k} \\
w_k^{t+1} &= \operatorname{CEG^+}(w_k^{t}) \quad \forall t = 0,...,\tau-1 \\
z^{k+1} &= (1-\lambda) z^{k} + \lambda w_k^{\tau}
\end{split}
\end{equation}
In the unconstrained case ($A\equiv 0$) this reduces to using the \ref{eq:EG+} algorithm of \citet{diakonikolas2021efficient} for the inner loop.
We have the following convergence guarantee.
\begin{thmbox}
\begin{correp}\label{thm:LA:CEG+}
Suppose \Cref{ass:S} holds.
Then $(z^{k})_{k \in \mathbb N}$ generated by \ref{eq:LA-CEG+} with $\lambda \in (0,1)$, $\gamma \in (\lfloor-2\rho\rfloor_+,\nicefrac 1L)$ and $\alpha \in (0, 1 + \tfrac{2\rho}{\gamma})$ converges to some $z^\star \in \zer S$.
\end{correp}
\end{thmbox}
\begin{appendixproof}
Quasi-nonexpansiveness of the operator $\operatorname{CEG^+}: \R^d \rightarrow \R^d$ follows from \Cref{thm:EG+:quasi-nonexpansive} provided $\alpha \in (0, 1 + \tfrac{2\rho}{\gamma})$ so \Cref{thm:LA:nonexpansive} applies.

It remains to verify that $\fix \operatorname{CEG^+} = \zer S$. 
This follows from
\begin{equation}
\tfrac{1}{\gamma} (\HC[][z] - \HC[][\z]) \in A(\z) + F(\z) = S(\z),
\end{equation}
and noticing that the stepsizes are positive, i.e. $\alpha > 0$ and $\gamma > 0$, which completes the proof.
\end{appendixproof}

\begin{toappendix}
\section{Analysis of CEG+}\label{app:EG+}
This section provides a simplified convergence proof of the CEG+ scheme proposed in \citet[Cor. 3.2]{pethick2022escaping} without going through adaptivity and a projected interpretation. 
We additionally provide convergence in terms of the residual $\|z^k-\bar z^k\|$.
The algorithm can be described with the following recursion
\begin{subequations}\label{eq:Main:AFBA}
\begin{align}
\z^{k} &= (\id + \gamma A)^{-1}(\HC[][z^k]) \label{eq:Main:AFBA:zbar} \\
z^{k+1} &= z^k - \alpha \left(\HC[][z^k] - \HC[][\z^k]\right) \label{eq:Main:AFBA:z}
\end{align}
\end{subequations}
where $\HC[][] = \id - \gamma F$. The \ref{eq:EG+} algorithm is obtained as a special case when $A\equiv 0$.

\begin{thm}\label{thm:EG+}
Suppose \Cref{ass:S} and $\gamma \in (\lfloor-2\rho\rfloor_+,\nicefrac 1L]$. 
Consider the sequence $(z^k)_{k\in \mathbb N}$ generated by \eqref{eq:Main:AFBA}.
Then, for all $z^\star \in \zer S$, it follows that
\begin{thmnum}
  \item the iterates $(z^k)_{k\in \mathbb N}$ satisfies \label{thm:EG+:quasi-nonexpansive}
   \begin{align*}
    \|z^{k+1}- z^\star\|^2 
        {}\leq{}&
    \|z^{k}- z^\star\|^2 
    - \alpha(1 + \tfrac{2\rho}{\gamma} - \alpha)\|\HC[][\z^k] - \HC[][z^k]\|^2
    - \alpha(1-\gamma^2L^2) \|\bar z^k-z^k\|^2,
    \end{align*}
    and in particular, $\operatorname{CEG^+}: \R^d \rightarrow \R^d$ in \eqref{eq:CEG+} is quasi-nonexpansive if $\alpha \in (0, 1 + \tfrac{2\rho}{\gamma})$.
  \item for $\alpha \in (0,1]$ and $\alpha < 1 + \tfrac{2\rho}{\gamma}$
  \begin{equation}
  \frac 1K \sum_{k=0}^{K-1}\|z^k - \bar z^k \|^2 \leq \frac{\|z^0 - z^\star \|^2}{\alpha(1-\gamma^2L^2)K}.
  \end{equation}
  \item for $\alpha \in (0,1)$ and $\alpha < 1 + \tfrac{2\rho}{\gamma}$
  \begin{equation}
  \frac 1K \sum_{k=0}^{K-1}\dist(0,S\z^k)^2 \leq \frac{\|z^0 - z^\star \|^2}{\alpha\gamma^2(1 + \tfrac{2\rho}{\gamma} - \alpha)K}.
  \end{equation}
\end{thmnum}
\end{thm}
\begin{proof}
By \(\nicefrac12\)-cocoercivity of $H=\id - \gamma F$ from \Cref{lm:Main:H:properties} we obtain
  \begin{align*}
    \langle \HC[][\z^k] - \HC[][z^k], z^k - z^\star \rangle
        {}={}&
    \langle \HC[][\z^k] - \HC[][z^k], \z^k - z^\star \rangle
        {}+{}
    \langle \HC[][\z^k] - \HC[][z^k], z^k - \z^k \rangle
    \\
        {}\leq{}&
    \langle \HC[][\z^k] - \HC[][z^k], \z^k - z^\star \rangle
    -\tfrac12\|\HC[][\z^k]-\HC[][z^k]\|^2 
    - \tfrac12 (1-\gamma^2L^2) \|\bar z^k-z^k\|^2
    \numberthis\label{eq:Main:AFBA:inner2}
\end{align*}
The update in \eqref{eq:Main:AFBA:z} yields
\begin{align*}
    \|z^{k+1}- z^\star\|^2 
        {}={}&
    \|z^{k}- z^\star\|^2 
        {}+{}
    \alpha^2\|\HC[][\z^k] - \HC[][z^k]\|^2
        {}+{}
    2\alpha\langle \HC[][\z^k] - \HC[][z^k], z^k - z^\star \rangle
    \\
    \dueto{\eqref{eq:Main:AFBA:inner2}}
        {}\leq{}&
    \|z^{k}- z^\star\|^2 
    -2\alpha\langle \HC[][z^k] - \HC[][\z^k], \z^k - z^\star \rangle \\
    &\quad - \alpha(1-\alpha)\|\HC[][\z^k] - \HC[][z^k]\|^2
    - \alpha(1-\gamma^2L^2) \|\bar z^k-z^k\|^2.
     \numberthis\label{eq:Main:AFBA:fejer}
\end{align*}
Noticing that both latter terms are negative. 
Observe that by \eqref{eq:Main:AFBA:zbar} we have
\[
    \tfrac{1}{\gamma} (\HC[][z^k] - \HC[][\z^k]) \in A(\z^k) + F(\z^k) = S(\z^k).
\]
Therefore, by cohypomonotonicity of $S=A+F$,
\begin{equation}\label{eq:Main:AFBA:monotone}
    \tfrac{1}{\gamma}\langle \HC[][z^k] - \HC[][\z^k], \z^k - z^\star \rangle
        {}\geq{}
    \rho\|\HC[][z^k] - \HC[][\z^k]\|^2. 
\end{equation}
and consequently \eqref{eq:Main:AFBA:fejer} leads to Fej\'er monotonicity,
\begin{align*}
    \|z^{k+1}- z^\star\|^2 
        {}\leq{}&
    \|z^{k}- z^\star\|^2 
    - \alpha(1 + \tfrac{2\rho}{\gamma} - \alpha)\|\HC[][\z^k] - \HC[][z^k]\|^2
    - \alpha(1-\gamma^2L^2) \|\bar z^k-z^k\|^2.
\end{align*}
By telescoping we obtain the two claims.
\end{proof}

\end{toappendix}

\begin{toappendix}
\section{Experiments}
\subsection{Simulations}

We repeat the synthetic examples for convenience below.

\begin{example}[PolarGame {\citep[Ex. 3(iii)]{pethick2022escaping}}]
\label{ex:polargame}
Consider $$Fz = \left(\psi(x,y)-y, \psi(y,x)+x\right),$$
where $\|z\|_\infty \leq \nicefrac{11}{10}$
and $\psi(x,y)=\frac{1}{16} a x (-1 + x^2 + y^2) (-9 + 16 x^2 + 16 y^2)$ with $a=\frac{1}{3}$.
\end{example}

\begin{example}[{Quadratic \citep[Ex. 5]{pethick2022escaping}}]
\label{ex:quadratic}
Consider,
\begin{equation}
\min_{x \in \mathbb R} \max_{y  \in \mathbb R} \phi(x,y) := axy + \frac{b}{2}x^2 - \frac{b}{2}y^2,
\end{equation}
where $a \in \mathbb R_+$ and $b \in \mathbb R$.
\end{example}
The problem constants in \Cref{ex:quadratic} can easily be computed as $\rho = \frac{b}{a^2+b^2}$ and $L = \sqrt{a^2+b^2}$.
We can rewrite \Cref{ex:quadratic} in terms of $L$ and $\rho$ by choosing
$a=\sqrt{L^2-L^4 \rho ^2}$ and $b = L^2 \rho$.

We provide below a slight generalization of the Forsaken example \citep[Example 5.2]{hsieh2021limits}, from which we derive another important case.

\begin{example}
\label{ex:forsaken}
Consider,
\begin{equation}
\min_{|x|\leq\nicefrac{3}{2}} \max_{|y|\leq\nicefrac{3}{2}} \phi(x,y):=x(y-a)+\psi(x)-\psi(y),
\end{equation}
where $\psi(z) = \frac{1}{4} z^{2}-\frac{1}{2} z^{4}+\frac{1}{6} z^{6}$ and $a\in \R$. 
We have the following important cases:
\begin{thmnum}
  \item for $a=0.45$ we recover Forsaken {\citep[Example 5.2]{hsieh2021limits}}.
  \item for $a=0.34$ we ensure that the first-order stationary point is a local Nash equilibrium (LNE), which is apparent from inspection of the Jacobian.
  We call this new example LNEForsaken.
\end{thmnum}
\end{example}

In both \Cref{ex:quadratic} and \Cref{ex:forsaken} the operator $F$ is defined as $Fz = (\nabla_x \phi(x,y), -\nabla_y \phi(x,y))$.
For \Cref{ex:forsaken} the Lookahead methods use $\tau=20$, $\lambda=0.2$ and $\gamma = \nicefrac{1}{L}\approx 0.08$ and (R)APP uses $\tau=10$, $\lambda=0.2$ and $\gamma = \nicefrac{4}{L}\approx 0.32$.
In \Cref{ex:polargame,ex:quadratic} we use $\gamma = \nicefrac{1}{L}$, $\lambda=0.1$ for \ref{eq:lookahead} and \ref{eq:EG+} with $\alpha=0.1$ for the latter. 
In the constrained examples $L$ refers to the Lipschitz constant constrained to the constraint set.

\begin{figure}[h]
\centering
\includegraphics[width=0.5\textwidth]{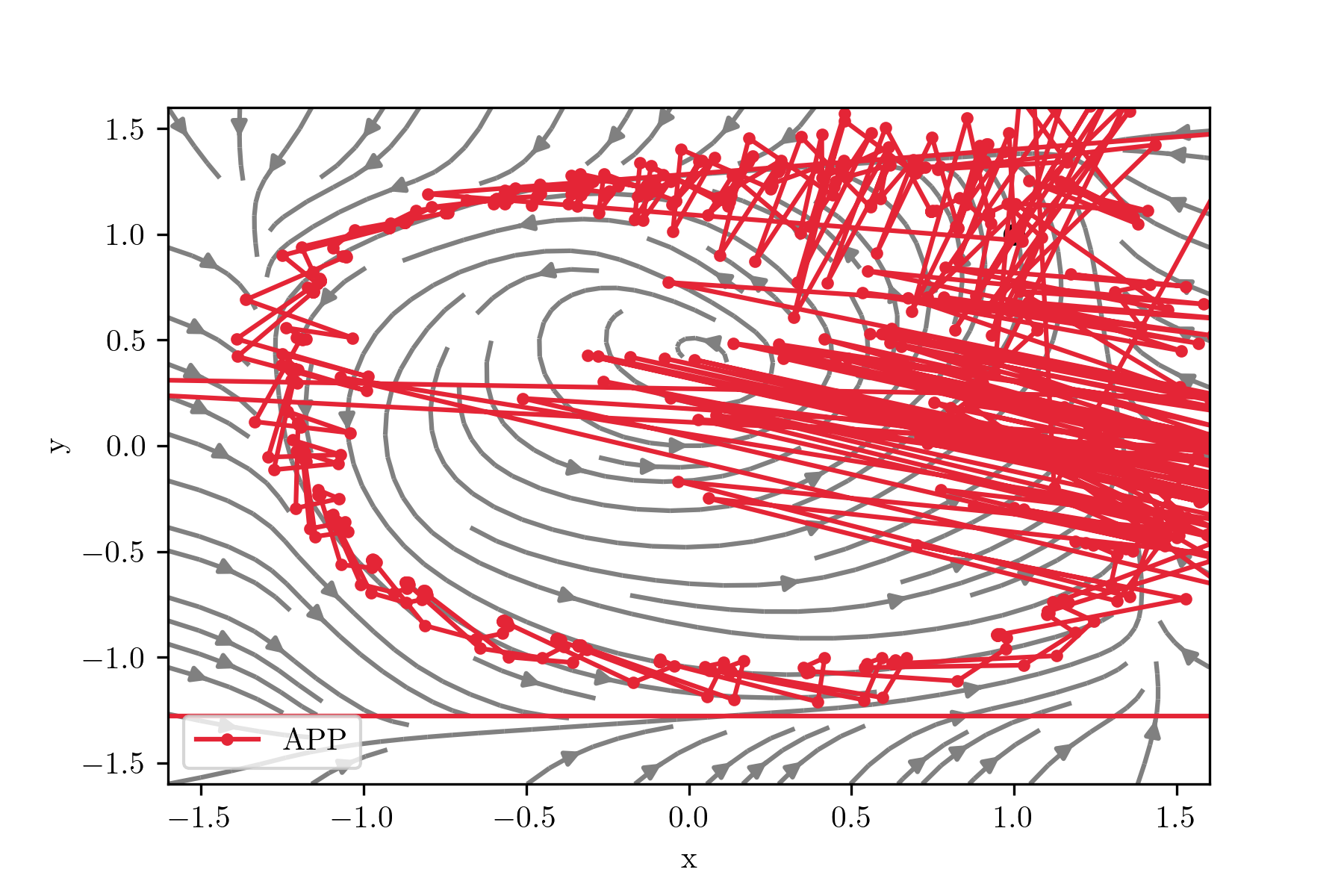}%
\caption{The iterates of APP associated with \Cref{fig:forsaken}.}
\label{fig:APPM:iterates}
\end{figure}

\begin{figure}[h]
\centering
\includegraphics[width=0.5\textwidth]{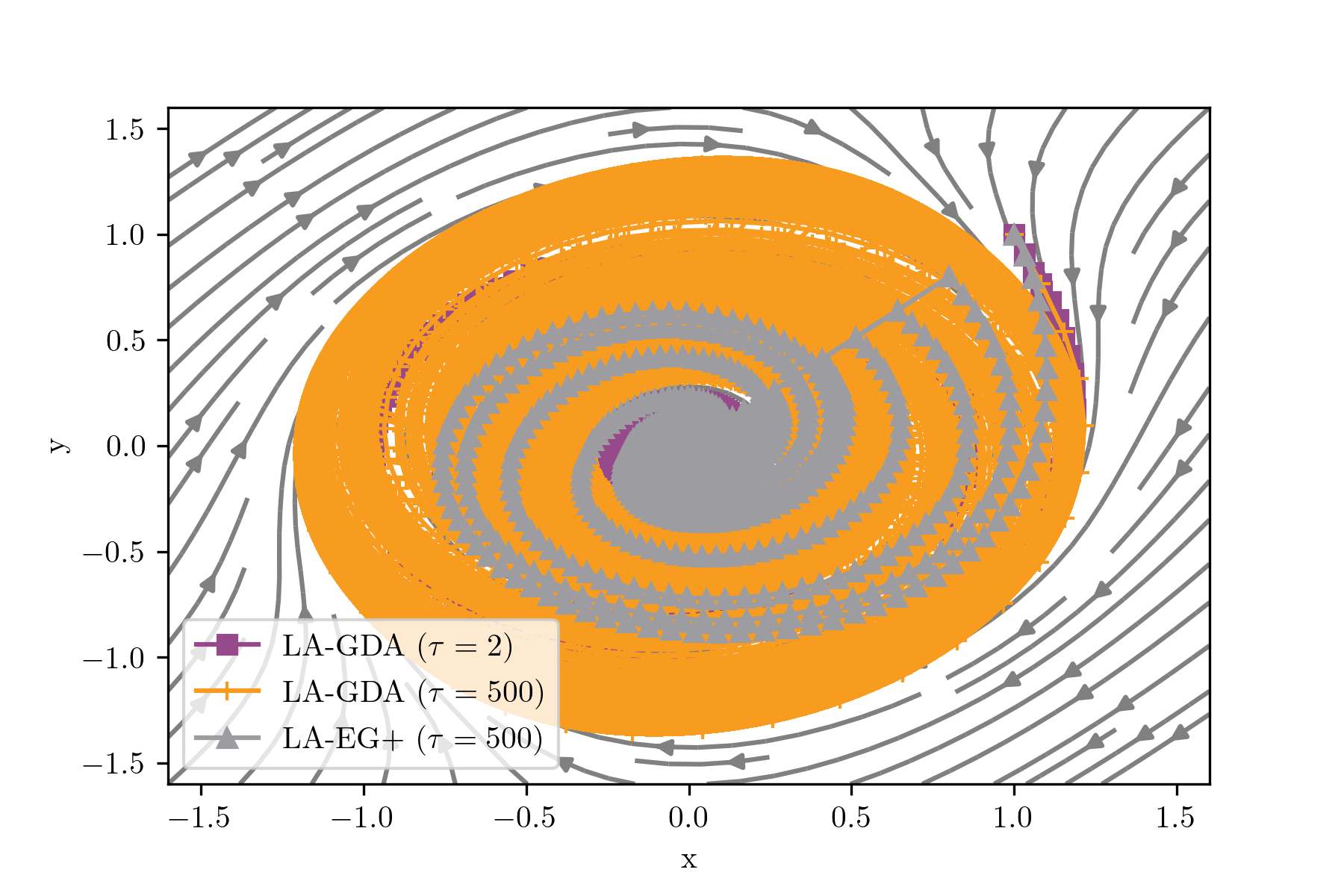}%
\includegraphics[width=0.5\textwidth]{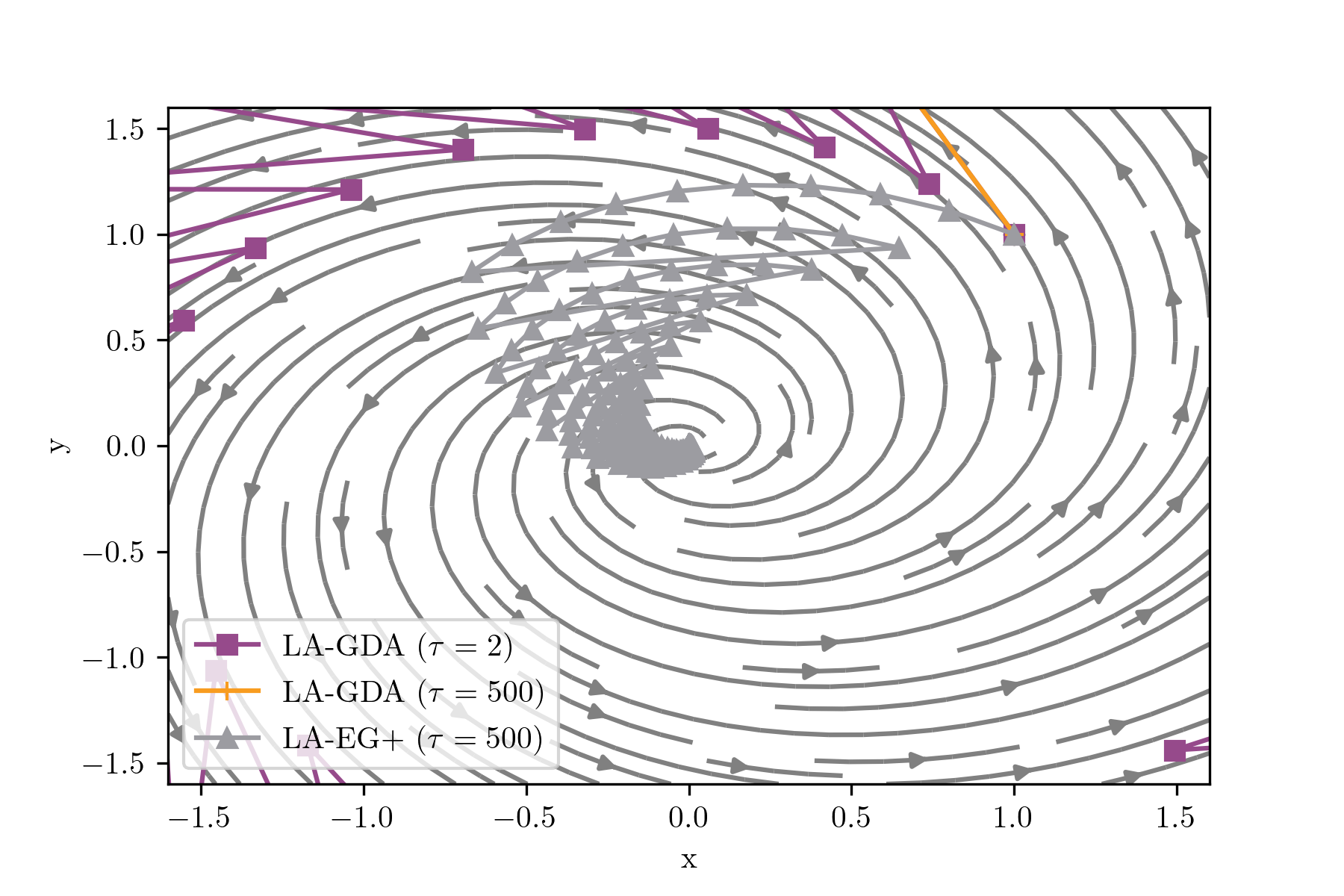}%
\caption{Iterates associated with \Cref{fig:LA:F}.}
\label{fig:LA:iterates}
\end{figure}

\subsection{Image generation}\label{app:imggen}

\paragraph{Architecture} The ResNet uses a 128-dimensional input space for the generator and spectral normalization for the discriminator (see \citet[Table 7]{chavdarova2020taming}).
The models' parameters are initialized using the Xavier initialization as suggested in \citet{miyato2018spectral}.

\paragraph{Optimizers}
All methods relies on stochastic gradients computed over a mini-batch. 
The discriminator and generator is updated in an alternating fashion.
We use the same variant of extragradient as \citet{chavdarova2020taming} uses in their implementation.
The variant only uses the extrapolated point of the \emph{opponent} in the update of the next iterate $(x^{k+1}, y^{k+1})$ as follows
\begin{equation}
\begin{split}
\bar x^{k} &= x^k - \gamma_1 \nabla \phi(x^k, y^k) \\
\bar y^{k} &= y^k + \gamma_2 \nabla \phi(x^k, y^k) \\
x^{k+1} &= x^k - \gamma_1 \nabla \phi(x^k, \bar y^k) \\
y^{k+1} &= y^k + \gamma_2 \nabla \phi(\bar x^k, y^k)
\end{split}
\end{equation}
Interestingly, we observed that the classical extragradient method (both a simultaneous and alternating variant) did not perform well under the hinge loss as used in the experiments.
We leave investigate of this for future work.

\paragraph{Evaluation}
We use the Fréchet inception distance (FID) \citep{heusel2017gans} evaluated on \num{50000} examples and the Inception score (ISC) \citep{salimans2016improved}.
For consistent and reproducible evaluation we use the \texttt{torch-fidelity} Python library \citep{obukhov2020torchfidelity} to compute the scores.
The mean and standard deviation is computed over 5 and 3 independent execution in \Cref{tab:Adam-based} and \Cref{tab:GDA-based}, respectively.

\paragraph{Compute time} Producing \Cref{tab:Adam-based} alone takes roughly $6 \text{ methods} \times 5 \text{ runs} \times 30\text{ hours} = 37.5 \text{ days}$ on a NVIDIA A100 GPU.

\begin{figure}[h]
\centering
\includegraphics[width=0.5\textwidth]{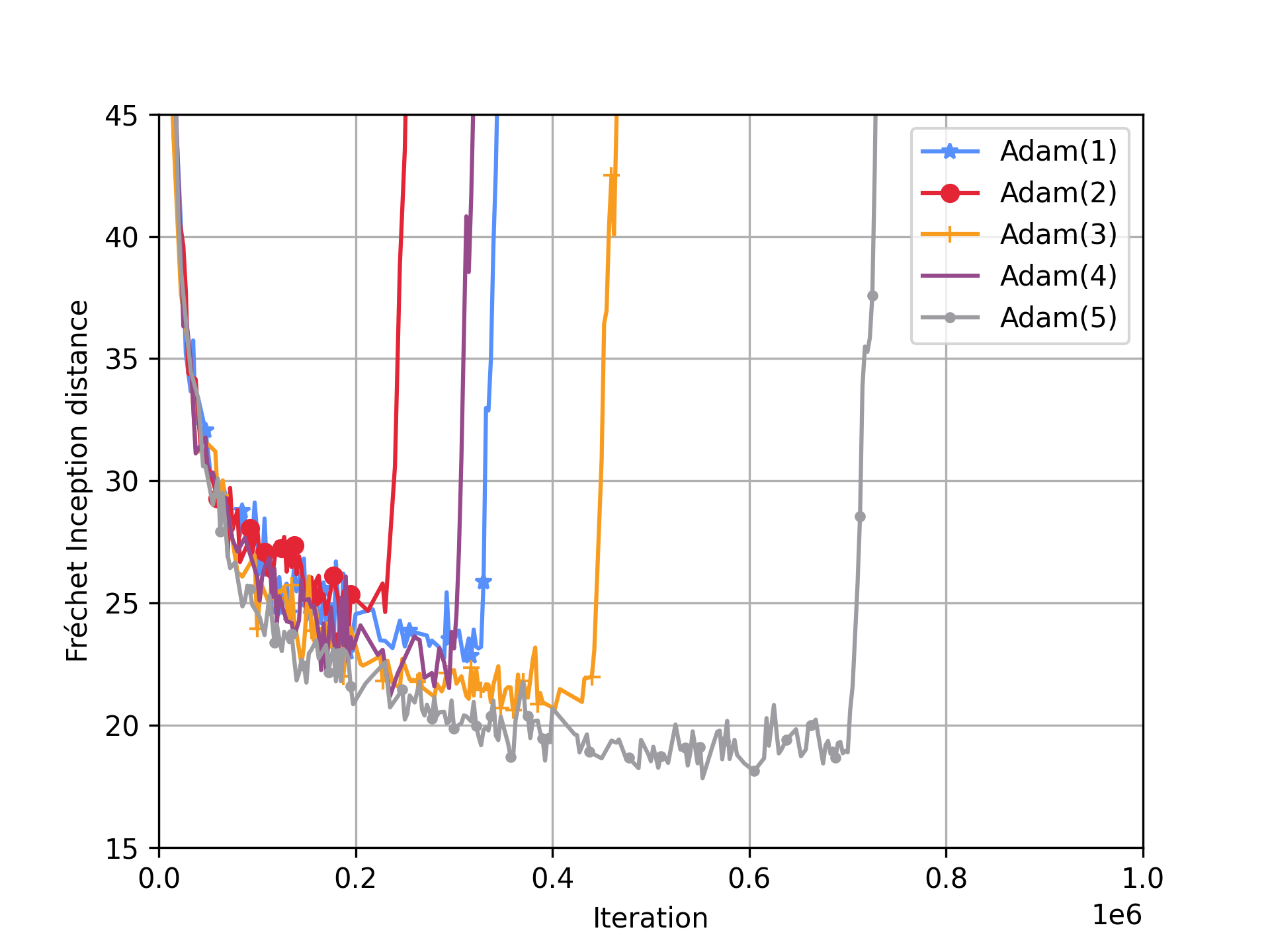}%
\includegraphics[width=0.5\textwidth]{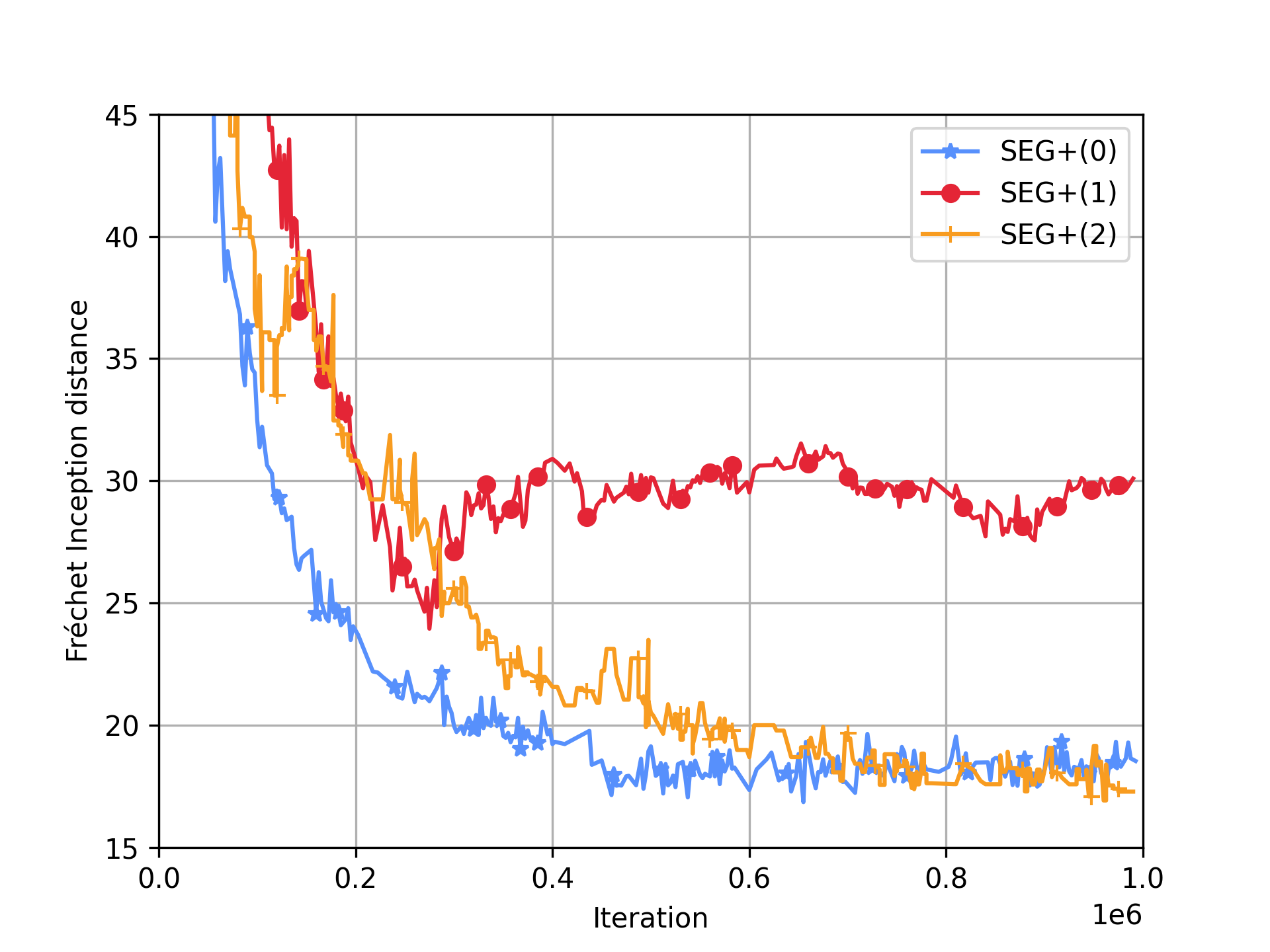}%
\caption{(left) Adam eventually diverges for all 5 runs. See \Cref{fig:adam} for comparison with Lookahead. (right) In contrast, EG+ increases stability (and thus avoids divergence), but in effect might also be stuck in a local (suboptimal) solution.
This explains the high variance and poor performance of EG+.  
By excluding the locally stuck run, EG+ achieves a FID of $16.88\pm0.05$ and a ISC of $8.0\pm0.02$, which is competitive even with the Lookahead-based methods.
}
\label{fig:Adam|EG+:all}
\end{figure}

\label{app:experiments}
\end{toappendix}
    \section{Experiments}\label{sec:experiments}
    This section demonstrates that linear interpolation can lead to an improvement over common baselines.

\paragraph{Synthetic examples}

\Cref{fig:forsaken,fig:LA:F} demonstrate \RAPP, \ref{eq:lookahead} and \ref{eq:LA-CEG+} on a host of nonmonotone problems (\citet[Ex. 5.2]{hsieh2021limits},
\citet[Ex. 3(iii)]{pethick2022escaping},
\citet[Ex. 5]{pethick2022escaping}).
See \Cref{app:experiments} for definitions and further details.

\paragraph{Image generation}
We replicate the experimental setup of \citet{chavdarova2020taming,miyato2018spectral}, which uses
hinge version of the non-saturated loss and a ResNet with spectral normalization for the discriminator (see \Cref{app:imggen} for details).
To evaluate the performance we rely on the commonly used Inception score (ISC) \citep{salimans2016improved} and Fréchet inception distance (FID) \citep{heusel2017gans} and report the best iterate.
We demonstrate the methods on the CIFAR10 dataset \citep{krizhevsky2009learning}.
The aim is \emph{not} to beat the state-of-the-art, but rather to complement the already exhaustive numerical evidence provided in \citet{chavdarova2020taming}.

For a fair \emph{computational} comparison we count the number of \emph{gradient computations} instead of iterations $k$ as in \citet{chavdarova2020taming}.
Maybe surprisingly, we find that the extrapolation methods such as EG and \RAPP still outperform the baseline, despite having fewer effective iterations.
\RAPP improves over EG, which suggest that it can be worthwhile to spend more computation on refining the updates at the cost of making fewer updates.

\begin{wrapfigure}{r}{6cm}
\centering
\vspace{-2.2em}
\includegraphics[width=0.4\textwidth]{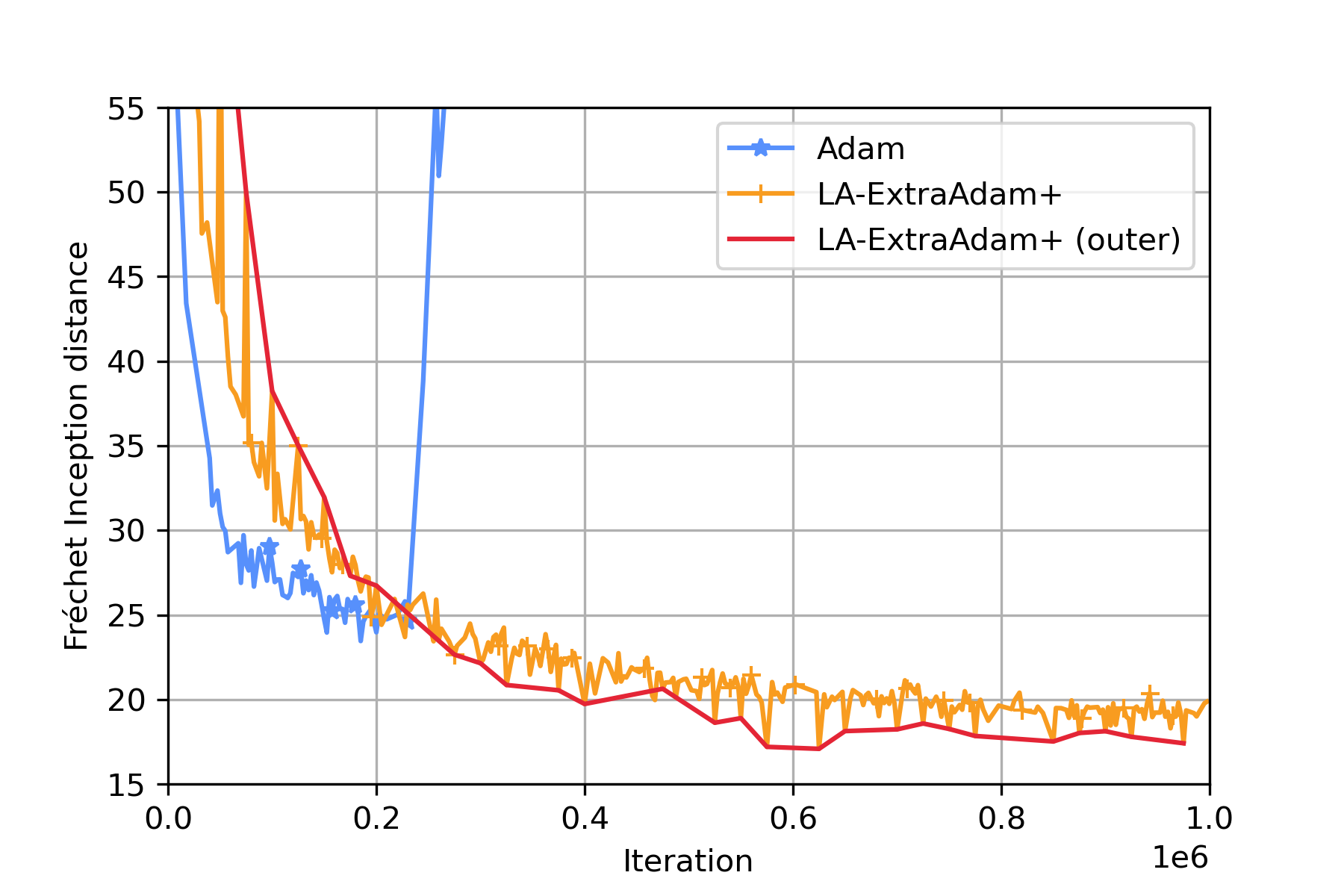}%
\caption{Adam eventually diverges on CIFAR10 while Lookahead is stable with the outer iterate enjoying superior performance.
\vspace{-2em}
}
\label{fig:adam}
\end{wrapfigure}

The first experiment we conduct matches the setting of \citet{chavdarova2020taming} by relying on the Adam optimizer and using and update ratio of $5{}:{}1$ between the discriminator and generator.
We find in \Cref{tab:Adam-based} that LA-ExtraAdam+ has the highest ISC (8.08) while LA-ExtraAdam has the lowest FID (15.88).
In contrast, we confirm that Adam is unstable while Lookahead prevents divergence as apparent from \Cref{fig:adam}, which is in agreement with \citet{chavdarova2020taming}.
In addition, the \emph{outer} loop of Lookahead achieves better empirical performance, which corroborate the theoretical result (cf. \Cref{rem:LA:outer}).
Notice that ExtraAdam+ has slow convergence (without Lookahead), which is possibly due to the $\nicefrac{1}{2}$-smaller stepsize.

We additionally simplify the setup by using GDA-based optimizers with an update ratio of $1{}:{}1$, which avoids the complexity of diagonal adaptation, gradient history and multiple steps of the discriminator as in the Adam-based experiments.
The results are found in \Cref{tab:GDA-based}.
The learning rates are tuned for GDA and we use those parameters fixed across all other methods.
Despite being tuned on GDA, we find that extragradient methods, Lookahead-based methods and \RAPP all \emph{still} outperform GDA in terms of FID.
The biggest improvement comes from the linear interpolation based methods Lookahead and \RAPP (see \Cref{fig:Adam|EG+:all} for further discussion on EG+).
Interesting, the Lookahead-based methods are roughly comparable with their Adam variants (\Cref{tab:Adam-based}) while GDA even performs better than Adam.
\looseness=-1

\begin{table}
\vspace{-2em}
\parbox[t]{.45\linewidth}{
    \centering
    \caption{Adam-based.
    The combination of Lookahead and extragradient-like methods performs the best.
    }
    \begin{tabular}{c|c|c}
\toprule
     & \textbf{FID} & \textbf{ISC}\\
\midrule                                                                 
Adam     & 21.04$\pm$2.20    & 7.61$\pm$0.15    \\
ExtraAdam    & 18.23$\pm$1.13   &  7.79$\pm$0.08    \\
ExtraAdam+    & 22.94$\pm$1.93   &  7.65$\pm$0.13   \\
\hline
LA-Adam   &  17.63$\pm$0.65    & 7.86$\pm$0.07    	\\
LA-ExtraAdam   & \bfseries{15.88$\pm$0.67}      & 7.97$\pm$0.12    \\
LA-ExtraAdam+   & 17.86$\pm$1.03     & \bfseries{8.08$\pm$0.15}    \\
\bottomrule
    \end{tabular}
    \label{tab:Adam-based}
}
\hfill
\parbox[t]{.45\linewidth}{
    \centering
    \caption{GDA-based.
    Both \RAPP and Lookahead increases the scores substantially.}
    \begin{tabular}{c|c|c}
\toprule
     & \textbf{FID} & \textbf{ISC}\\
\midrule                                                                 
GDA     & 19.36$\pm$0.08     & 7.84$\pm$0.07    \\
EG    & 18.94$\pm$0.60  &  7.84$\pm$0.02    \\
EG+    & 19.35$\pm$4.28   &  7.74$\pm$0.44   \\
\hline
LA-GDA   & \bfseries{16.87$\pm$0.18}    & \bfseries{8.01$\pm$0.08}   	\\
LA-EG   &  16.91$\pm$0.66      & 7.97$\pm$0.12     \\
LA-EG+   & 17.20$\pm$0.44      & 7.94$\pm$0.11    \\
RAPP   & 17.76$\pm$0.82     & 7.98$\pm$0.08    \\
\bottomrule
    \end{tabular}
    \label{tab:GDA-based}
}
\vspace{-1em}
\end{table}

\section{Conclusion \& limitations}\label{sec:conclusion}

    We have precisely characterized the stabilizing effect of linear interpolation by analyzing it under cohypomonotonicity.
    We proved last iterate convergence rates for our proposed method \RAPP.
    The algorithm is double-looped, which introduces a log factor in the rate as mentioned in \Cref{rem:averagedAPPM:last}. 
    It thus remains open whether last iterate is possible using only $\tau=2$ inner iterations (for which \RAPP reduces to \ref{eq:EG+} in the unconstrained case).
    By replacing the inner solver we subsequently rediscovered and analyzed Lookahead using nonexpansive operators. 
    In that regard, we have only dealt with compositions of operators.
    It would be interesting to further extend the idea to understanding and developing both Federated Averaging and the meta-learning algorithm Reptile (of which Lookahead can be seen as a single client and single task instance respectively), which we leave for future work.

\section{Acknowledgements}\label{sec:acknowledgement}
Thanks to Ahmet Alacaoglu, Weiqiang Zheng, Leello Dadi and Stratis Skoulakis for helpful discussion.
This work was supported by Google. %
This work was supported by Hasler Foundation Program: Hasler Responsible AI (project number 21043). %
This work was supported by the Swiss National Science Foundation (SNSF) under grant number 200021\_205011. %
This project has received funding from the European Research Council (ERC) under the European Union's Horizon 2020 research and innovation programme (grant agreement n° 725594 - time-data). %

\begin{toappendix}
\newpage
\subsubsection{Hyperparameters}

\begin{table}[h]
    \centering
    \caption{Training Hyperparameters for Adam-based experiments on CIFAR10}
\begin{adjustbox}{center}
    \begin{tabular}{c|c|c|c|c|c|c}
\toprule
\textbf{Hyperparameter} & \textbf{Adam} & \textbf{LA-Adam}  & \textbf{ExtraAdam+} & \textbf{LA-ExtraAdam+} & \textbf{ExtraAdam} & \textbf{LA-ExtraAdam}\\
\midrule                                                                 
lrD     & 2e-4    & 2e-4   & 2e-4   & 2e-4   & 2e-4   & 2e-4  \\
lrG   & 2e-4    & 2e-4   & 2e-4   & 2e-4   & 2e-4   & 2e-4  	\\
Batch Size    & 128   &  128   &  128  &  128   & 128    & 128	\\
$\beta_1$   & 0.0      & 0.0   & 0.0   & 0.0   & 0.0     & 0.0    \\
D-steps  & 5      & 5  & 5  & 5  & 5   & 5  			\\
Lookahead $\tau$  &     & 5   &    &  5000     &    &  5000  \\
Lookahead $\lambda$  &     & 0.5   &    &  0.5     &    &  0.5  \\
EG+ $\alpha$  &   &    & 0.5    & 0.5    &        &  \\
\bottomrule
    \end{tabular}
    \end{adjustbox}
    \label{tab:hp-Adam}
\end{table}

\begin{table}[h]
    \centering
    \caption{Training Hyperparameters for GDA-based experiments on CIFAR10}
    \begin{tabular}{c|c|c|c|c|c|c|c}
\toprule
\textbf{Hyperparameter} & \textbf{GDA} & \textbf{LA-GDA} &\textbf{EG+} & \textbf{LA-EG+} &\textbf{EG} & \textbf{LA-EG} & \textbf{RAPP}\\
\midrule                                                                 
lrD     & 0.1    &  0.1    &  0.1   & 0.1  &  0.1   & 0.1 & 0.1   \\
lrG   & 0.02   & 0.02  & 0.02  & 0.02   & 0.02  & 0.02   & 0.02 	\\
Batch Size    & 128   &  128   &  128  &  128 &  128  &  128  &  128	\\
D-steps  & 1      & 1  & 1  & 1 	 & 1  & 1 	& 1 		\\
Lookahead $\tau$  &     & 5000   &     &  5000  &     &  5000   & \\
Lookahead $\lambda$    &   &  0.5     &    &  0.5  &    &  0.5   & \\
EG+ $\alpha$  &   &    & 0.5    & 0.5    &        &  \\
RAPP $\tau$  &     &    &     &    &     &    &  3 \\
RAPP $\lambda$    &   &       &    &   &    &     &  0.9\\
\bottomrule
    \end{tabular}
    \label{tab:hp-SGD}
\end{table}
\label{app:hyperparameters}
\end{toappendix}

\bibliographystyle{icml2023}
\bibliography{ref.bib}

\begin{thebibliography}{53}
\providecommand{\natexlab}[1]{#1}
\providecommand{\url}[1]{\texttt{#1}}
\expandafter\ifx\csname urlstyle\endcsname\relax
  \providecommand{\doi}[1]{doi: #1}\else
  \providecommand{\doi}{doi: \begingroup \urlstyle{rm}\Url}\fi

\bibitem[Alacaoglu et~al.(2024)Alacaoglu, Kim, and
  Wright]{alacaoglu2024extending}
Alacaoglu, A., Kim, D., and Wright, S.~J.
\newblock Extending the reach of first-order algorithms for nonconvex min-max
  problems with cohypomonotonicity.
\newblock \emph{arXiv preprint arXiv:2402.05071}, 2024.

\bibitem[Arjevani et~al.(2015)Arjevani, Shalev-Shwartz, and
  Shamir]{arjevani2015lower}
Arjevani, Y., Shalev-Shwartz, S., and Shamir, O.
\newblock On lower and upper bounds for smooth and strongly convex optimization
  problems.
\newblock \emph{arXiv preprint arXiv:1503.06833}, 2015.

\bibitem[Banach(1922)]{banach1922operations}
Banach, S.
\newblock Sur les op{\'e}rations dans les ensembles abstraits et leur
  application aux {\'e}quations int{\'e}grales.
\newblock \emph{Fund. math}, 3\penalty0 (1):\penalty0 133--181, 1922.

\bibitem[Bartz et~al.(2022)Bartz, Dao, and Phan]{bartz2022conical}
Bartz, S., Dao, M.~N., and Phan, H.~M.
\newblock Conical averagedness and convergence analysis of fixed point
  algorithms.
\newblock \emph{Journal of Global Optimization}, 82\penalty0 (2):\penalty0
  351--373, 2022.

\bibitem[Bauschke \& Combettes(2017)Bauschke and Combettes]{Bauschke2017Convex}
Bauschke, H.~H. and Combettes, P.~L.
\newblock \emph{Convex analysis and monotone operator theory in {H}ilbert
  spaces}.
\newblock CMS Books in Mathematics. Springer, 2017.
\newblock ISBN 978-3-319-48310-8.

\bibitem[Bauschke et~al.(2021)Bauschke, Moursi, and
  Wang]{bauschke2021generalized}
Bauschke, H.~H., Moursi, W.~M., and Wang, X.
\newblock Generalized monotone operators and their averaged resolvents.
\newblock \emph{Mathematical Programming}, 189\penalty0 (1):\penalty0 55--74,
  2021.

\bibitem[Bertsekas(2011)]{bertsekas2011incremental}
Bertsekas, D.~P.
\newblock Incremental proximal methods for large scale convex optimization.
\newblock \emph{Mathematical programming}, 129\penalty0 (2):\penalty0 163--195,
  2011.

\bibitem[Bianchi(2015)]{bianchi2015convergence}
Bianchi, P.
\newblock On the convergence of a stochastic proximal point algorithm.
\newblock In \emph{CAMSAP}, 2015.

\bibitem[Bravo \& Cominetti(2022)Bravo and Cominetti]{bravo2022stochastic}
Bravo, M. and Cominetti, R.
\newblock Stochastic fixed-point iterations for nonexpansive maps:
  {Convergence} and error bounds.
\newblock \emph{arXiv preprint arXiv:2208.04193}, 2022.

\bibitem[Br{\'e}zis \& Lions(1978)Br{\'e}zis and Lions]{brezis1978produits}
Br{\'e}zis, H. and Lions, P.~L.
\newblock Produits infinis de r{\'e}solvantes.
\newblock \emph{Israel Journal of Mathematics}, 29\penalty0 (4):\penalty0
  329--345, 1978.

\bibitem[Cai et~al.(2022)Cai, Oikonomou, and Zheng]{cai2022accelerated}
Cai, Y., Oikonomou, A., and Zheng, W.
\newblock Accelerated algorithms for monotone inclusions and constrained
  nonconvex-nonconcave min-max optimization.
\newblock \emph{arXiv preprint arXiv:2206.05248}, 2022.

\bibitem[Cevher et~al.(2023)Cevher, Piliouras, Sim, and
  Skoulakis]{https://doi.org/10.48550/arxiv.2301.03931}
Cevher, V., Piliouras, G., Sim, R., and Skoulakis, S.
\newblock Min-max optimization made simple: Approximating the proximal point
  method via contraction maps, 2023.
\newblock URL \url{https://arxiv.org/abs/2301.03931}.

\bibitem[Chavdarova et~al.(2020)Chavdarova, Pagliardini, Stich, Fleuret, and
  Jaggi]{chavdarova2020taming}
Chavdarova, T., Pagliardini, M., Stich, S.~U., Fleuret, F., and Jaggi, M.
\newblock Taming {GANs} with {Lookahead}-{Minmax}.
\newblock \emph{arXiv preprint arXiv:2006.14567}, 2020.

\bibitem[Combettes(2001)]{combettes2001quasi}
Combettes, P.~L.
\newblock {Quasi}-{Fej{\'e}rian} analysis of some optimization algorithms.
\newblock In \emph{Studies in Computational Mathematics}, volume~8, pp.\
  115--152. Elsevier, 2001.

\bibitem[Combettes \& Pennanen(2004)Combettes and
  Pennanen]{combettes2004proximal}
Combettes, P.~L. and Pennanen, T.
\newblock Proximal methods for cohypomonotone operators.
\newblock \emph{SIAM journal on control and optimization}, 43\penalty0
  (2):\penalty0 731--742, 2004.

\bibitem[Diakonikolas(2020)]{diakonikolas2020halpern}
Diakonikolas, J.
\newblock Halpern iteration for near-optimal and parameter-free monotone
  inclusion and strong solutions to variational inequalities.
\newblock In \emph{Conference on Learning Theory}, pp.\  1428--1451. PMLR,
  2020.

\bibitem[Diakonikolas et~al.(2021)Diakonikolas, Daskalakis, and
  Jordan]{diakonikolas2021efficient}
Diakonikolas, J., Daskalakis, C., and Jordan, M.~I.
\newblock Efficient methods for structured nonconvex-nonconcave min-max
  optimization.
\newblock In \emph{International Conference on Artificial Intelligence and
  Statistics}, pp.\  2746--2754. PMLR, 2021.

\bibitem[Eckstein \& Bertsekas(1992)Eckstein and
  Bertsekas]{eckstein1992douglas}
Eckstein, J. and Bertsekas, D.~P.
\newblock On the {Douglas—Rachford} splitting method and the proximal point
  algorithm for maximal monotone operators.
\newblock \emph{Mathematical Programming}, 55\penalty0 (1):\penalty0 293--318,
  1992.

\bibitem[Gidel et~al.(2018)Gidel, Berard, Vignoud, Vincent, and
  Lacoste-Julien]{gidel2018variational}
Gidel, G., Berard, H., Vignoud, G., Vincent, P., and Lacoste-Julien, S.
\newblock A variational inequality perspective on generative adversarial
  networks.
\newblock \emph{arXiv preprint arXiv:1802.10551}, 2018.

\bibitem[Golowich et~al.(2020)Golowich, Pattathil, Daskalakis, and
  Ozdaglar]{golowich2020last}
Golowich, N., Pattathil, S., Daskalakis, C., and Ozdaglar, A.
\newblock Last iterate is slower than averaged iterate in smooth convex-concave
  saddle point problems.
\newblock In \emph{Conference on Learning Theory}, pp.\  1758--1784. PMLR,
  2020.

\bibitem[Gorbunov et~al.(2022{\natexlab{a}})Gorbunov, Loizou, and
  Gidel]{gorbunov2022extragradient}
Gorbunov, E., Loizou, N., and Gidel, G.
\newblock Extragradient method: O (1/k) last-iterate convergence for monotone
  variational inequalities and connections with cocoercivity.
\newblock In \emph{International Conference on Artificial Intelligence and
  Statistics}, pp.\  366--402. PMLR, 2022{\natexlab{a}}.

\bibitem[Gorbunov et~al.(2022{\natexlab{b}})Gorbunov, Taylor, Horv{\'a}th, and
  Gidel]{gorbunov2022convergence}
Gorbunov, E., Taylor, A., Horv{\'a}th, S., and Gidel, G.
\newblock Convergence of proximal point and extragradient-based methods beyond
  monotonicity: the case of negative comonotonicity.
\newblock \emph{arXiv preprint arXiv:2210.13831}, 2022{\natexlab{b}}.

\bibitem[Grimmer et~al.(2022)Grimmer, Lu, Worah, and
  Mirrokni]{grimmer2022landscape}
Grimmer, B., Lu, H., Worah, P., and Mirrokni, V.
\newblock The landscape of the proximal point method for nonconvex--nonconcave
  minimax optimization.
\newblock \emph{Mathematical Programming}, pp.\  1--35, 2022.

\bibitem[Ha \& Kim(2022)Ha and Kim]{ha2022convergence}
Ha, J. and Kim, G.
\newblock On convergence of {Lookahead} in smooth games.
\newblock In \emph{International Conference on Artificial Intelligence and
  Statistics}, pp.\  4659--4684. PMLR, 2022.

\bibitem[Halpern(1967)]{halpern1967fixed}
Halpern, B.
\newblock Fixed points of nonexpanding maps.
\newblock 1967.

\bibitem[Heusel et~al.(2017)Heusel, Ramsauer, Unterthiner, Nessler, and
  Hochreiter]{heusel2017gans}
Heusel, M., Ramsauer, H., Unterthiner, T., Nessler, B., and Hochreiter, S.
\newblock {GANs} trained by a two time-scale update rule converge to a local
  {Nash} equilibrium.
\newblock \emph{Advances in neural information processing systems}, 30, 2017.

\bibitem[Hsieh et~al.(2021)Hsieh, Mertikopoulos, and Cevher]{hsieh2021limits}
Hsieh, Y.-P., Mertikopoulos, P., and Cevher, V.
\newblock The limits of min-max optimization algorithms: {Convergence} to
  spurious non-critical sets.
\newblock In \emph{International Conference on Machine Learning}, pp.\
  4337--4348. PMLR, 2021.

\bibitem[Iusem et~al.(2003)Iusem, Pennanen, and Svaiter]{iusem2003inexact}
Iusem, A.~N., Pennanen, T., and Svaiter, B.~F.
\newblock Inexact variants of the proximal point algorithm without
  monotonicity.
\newblock \emph{SIAM Journal on Optimization}, 13\penalty0 (4):\penalty0
  1080--1097, 2003.

\bibitem[Korpelevich(1977)]{korpelevich1977extragradient}
Korpelevich, G.
\newblock Extragradient method for finding saddle points and other problems.
\newblock \emph{Matekon}, 13\penalty0 (4):\penalty0 35--49, 1977.

\bibitem[Krizhevsky et~al.(2009)Krizhevsky, Hinton,
  et~al.]{krizhevsky2009learning}
Krizhevsky, A., Hinton, G., et~al.
\newblock Learning multiple layers of features from tiny images.
\newblock 2009.

\bibitem[Lee \& Kim(2021)Lee and Kim]{lee2021fast}
Lee, S. and Kim, D.
\newblock Fast extra gradient methods for smooth structured
  nonconvex-nonconcave minimax problems.
\newblock \emph{Advances in Neural Information Processing Systems},
  34:\penalty0 22588--22600, 2021.

\bibitem[Lieder(2021)]{lieder2021convergence}
Lieder, F.
\newblock On the convergence rate of the halpern-iteration.
\newblock \emph{Optimization letters}, 15\penalty0 (2):\penalty0 405--418,
  2021.

\bibitem[Luo \& Tran-Dinh()Luo and Tran-Dinh]{luoextragradient}
Luo, Y. and Tran-Dinh, Q.
\newblock Extragradient-type methods for co-monotone root-finding problems.

\bibitem[McMahan et~al.(2017)McMahan, Moore, Ramage, Hampson, and
  y~Arcas]{mcmahan2017communication}
McMahan, B., Moore, E., Ramage, D., Hampson, S., and y~Arcas, B.~A.
\newblock Communication-efficient learning of deep networks from decentralized
  data.
\newblock In \emph{Artificial intelligence and statistics}, pp.\  1273--1282.
  PMLR, 2017.

\bibitem[Miyato et~al.(2018)Miyato, Kataoka, Koyama, and
  Yoshida]{miyato2018spectral}
Miyato, T., Kataoka, T., Koyama, M., and Yoshida, Y.
\newblock Spectral normalization for generative adversarial networks.
\newblock \emph{arXiv preprint arXiv:1802.05957}, 2018.

\bibitem[Nemirovski(2004)]{nemirovski2004prox}
Nemirovski, A.
\newblock Prox-method with rate of convergence o (1/t) for variational
  inequalities with {Lipschitz} continuous monotone operators and smooth
  convex-concave saddle point problems.
\newblock \emph{SIAM Journal on Optimization}, 15\penalty0 (1):\penalty0
  229--251, 2004.

\bibitem[Nichol et~al.(2018)Nichol, Achiam, and Schulman]{nichol2018first}
Nichol, A., Achiam, J., and Schulman, J.
\newblock On first-order meta-learning algorithms.
\newblock \emph{arXiv preprint arXiv:1803.02999}, 2018.

\bibitem[Obukhov et~al.(2020)Obukhov, Seitzer, Wu, Zhydenko, Kyl, and
  Lin]{obukhov2020torchfidelity}
Obukhov, A., Seitzer, M., Wu, P.-W., Zhydenko, S., Kyl, J., and Lin, E. Y.-J.
\newblock High-fidelity performance metrics for generative models in {PyTorch},
  2020.
\newblock URL \url{https://github.com/toshas/torch-fidelity}.
\newblock Version: 0.3.0, DOI: 10.5281/zenodo.4957738.

\bibitem[Opial(1967)]{opial1967weak}
Opial, Z.
\newblock Weak convergence of the sequence of successive approximations for
  nonexpansive mappings.
\newblock \emph{Bulletin of the American Mathematical Society}, 73\penalty0
  (4):\penalty0 591--597, 1967.

\bibitem[Patrascu \& Irofti(2021)Patrascu and Irofti]{patrascu2021stochastic}
Patrascu, A. and Irofti, P.
\newblock Stochastic proximal splitting algorithm for composite minimization.
\newblock \emph{Optimization Letters}, 15\penalty0 (6):\penalty0 2255--2273,
  2021.

\bibitem[Patrascu \& Necoara(2017)Patrascu and
  Necoara]{patrascu2017nonasymptotic}
Patrascu, A. and Necoara, I.
\newblock Nonasymptotic convergence of stochastic proximal point methods for
  constrained convex optimization.
\newblock \emph{The Journal of Machine Learning Research}, 18\penalty0
  (1):\penalty0 7204--7245, 2017.

\bibitem[Pethick et~al.(2022)Pethick, Latafat, Patrinos, Fercoq, and
  Cevher]{pethick2022escaping}
Pethick, T., Latafat, P., Patrinos, P., Fercoq, O., and Cevher, V.
\newblock Escaping limit cycles: {Global} convergence for constrained
  nonconvex-nonconcave minimax problems.
\newblock In \emph{International Conference on Learning Representations}, 2022.

\bibitem[Pethick et~al.(2023)Pethick, Fercoq, Latafat, Patrinos, and
  Cevher]{anonymous2023solving}
Pethick, T., Fercoq, O., Latafat, P., Patrinos, P., and Cevher, V.
\newblock Solving stochastic weak {Minty} variational inequalities without
  increasing batch size.
\newblock In \emph{The Eleventh International Conference on Learning
  Representations}, 2023.

\bibitem[Pushkin \& Barba(2021)Pushkin and Barba]{pushkin2021multilayer}
Pushkin, D. and Barba, L.
\newblock Multilayer {Lookahead}: a nested version of {Lookahead}.
\newblock \emph{arXiv preprint arXiv:2110.14254}, 2021.

\bibitem[Rockafellar(1976)]{rockafellar1976monotone}
Rockafellar, R.~T.
\newblock Monotone operators and the proximal point algorithm.
\newblock \emph{SIAM journal on control and optimization}, 14\penalty0
  (5):\penalty0 877--898, 1976.

\bibitem[Salimans et~al.(2016)Salimans, Goodfellow, Zaremba, Cheung, Radford,
  and Chen]{salimans2016improved}
Salimans, T., Goodfellow, I., Zaremba, W., Cheung, V., Radford, A., and Chen,
  X.
\newblock Improved techniques for training {GANs}.
\newblock \emph{Advances in neural information processing systems}, 29, 2016.

\bibitem[Toulis et~al.(2015)Toulis, Horel, and Airoldi]{toulis2015proximal}
Toulis, P., Horel, T., and Airoldi, E.~M.
\newblock The proximal {Robbins-Monro} method.
\newblock \emph{arXiv preprint arXiv:1510.00967}, 2015.

\bibitem[Toulis et~al.(2016)Toulis, Tran, and Airoldi]{toulis2016towards}
Toulis, P., Tran, D., and Airoldi, E.
\newblock Towards stability and optimality in stochastic gradient descent.
\newblock In \emph{Artificial Intelligence and Statistics}, pp.\  1290--1298.
  PMLR, 2016.

\bibitem[Tseng(1991)]{tseng1991applications}
Tseng, P.
\newblock Applications of a splitting algorithm to decomposition in convex
  programming and variational inequalities.
\newblock \emph{SIAM Journal on Control and Optimization}, 29\penalty0
  (1):\penalty0 119--138, 1991.

\bibitem[Wang et~al.(2020)Wang, Tantia, Ballas, and Rabbat]{wang2020lookahead}
Wang, J., Tantia, V., Ballas, N., and Rabbat, M.
\newblock Lookahead converges to stationary points of smooth non-convex
  functions.
\newblock In \emph{ICASSP 2020-2020 IEEE International Conference on Acoustics,
  Speech and Signal Processing (ICASSP)}, pp.\  8604--8608. IEEE, 2020.

\bibitem[Yoon \& Ryu(2021)Yoon and Ryu]{yoon2021accelerated}
Yoon, T. and Ryu, E.~K.
\newblock Accelerated algorithms for smooth convex-concave minimax problems
  with o (1/k\^{} 2) rate on squared gradient norm.
\newblock In \emph{International Conference on Machine Learning}, pp.\
  12098--12109. PMLR, 2021.

\bibitem[Zhang et~al.(2019)Zhang, Lucas, Ba, and Hinton]{zhang2019lookahead}
Zhang, M., Lucas, J., Ba, J., and Hinton, G.~E.
\newblock Lookahead optimizer: k steps forward, 1 step back.
\newblock \emph{Advances in neural information processing systems}, 32, 2019.

\bibitem[Zhou et~al.(2021)Zhou, Yan, Yuan, Feng, and Yan]{zhou2021towards}
Zhou, P., Yan, H., Yuan, X., Feng, J., and Yan, S.
\newblock Towards understanding why lookahead generalizes better than {SGD} and
  beyond.
\newblock \emph{Advances in Neural Information Processing Systems},
  34:\penalty0 27290--27304, 2021.

\end{thebibliography}

\newpage
\appendix
\onecolumn

\begin{center}
\vspace{7pt}
{\Large \fontseries{bx}\selectfont Appendix}
\end{center}

\renewcommand{\contentsname}{Table of Contents}
\etocdepthtag.toc{mtappendix}
\etocsettagdepth{mtchapter}{none}
\etocsettagdepth{mtappendix}{subsection}
\tableofcontents

\newpage

\end{document}